\documentclass[11pt,letterpaper]{article}
\usepackage{helvet}
\usepackage{courier}
\usepackage{flushend}
\usepackage{tikz}
\usepackage{pgfplots}
\usepackage{bm}
\usepackage{stmaryrd}
\newcount\Comments  % 0 suppresses notes to selves in text
\usepackage{amsmath}
\usepackage{amssymb}

\usepackage{amsthm}

\usepackage{float}
\usepackage{graphicx}
\usepackage{sidecap}
\usepackage{algorithm}
\usepackage{algpseudocode}
\usepackage{url}
\usepackage{multirow}

\usepackage{enumerate}
\usepackage{paralist}
\usepackage{subfigure}
\usepackage{wrapfig}
\usepackage[square]{natbib}
\usepackage{nicefrac}
\usepackage{mathrsfs}
\usepackage{xcolor}
\usepackage{float}
\usepackage{geometry}

\newtheorem{theorem}{Theorem}
\newtheorem{corollary}{Corollary}
\newtheorem{lemma}{Lemma}

\newtheorem{proposition}[theorem]{Proposition}

\newcommand{\kibitz}[2]{\ifnum\Comments=1{\color{#1}{#2}}\fi}
\newcommand{\rmr}[1]{\kibitz{blue}{[RESHEF:#1]}}
\newcommand{\dcp}[1]{\kibitz{orange}{[DAVID:#1]}}
\newcommand{\ccw}[1]{\kibitz{teal}{[CCW:#1]}}

\def\ol{\overline}
\def\ul{\underline}

\def\cite{\citep}
\def\shortcite{\citeyearpar}

\newcommand{\labeq}[2]{   % labeled equation
\begin{equation}
\label{eq:#1}
#2
\end{equation}}
\newcommand{\tup}[1]{\left\langle #1\right\rangle}

\renewcommand{\vec}[1]{\mathbf{#1}}

\def\eps{\epsilon}
\newcommand\PW[1]{PW$^{(#1)}$}

\newcommand\newpar[1]{\vspace{-0mm}\paragraph{#1}}

\def\subsec{\vspace{-0mm}\subsection}
\def\I{{\cal I}}
\newcommand{\ignore}[1]{}

\begin{document}

\title{On Sex, Evolution, and the Multiplicative Weights Update Algorithm}

%\numberofauthors{2} \author{
%\alignauthor Reshef Meir\\
%\affaddr{Harvard University, USA}\\
%\email{rmeir@seas.harvard.edu} \and
 %\alignauthor David Parkes\\
%\affaddr{Harvard University, USA}\\
%\email{parkes@eecs.harvard.edu}}

%\author{Submission \#231}

\author{Reshef Meir$^*$ and David Parkes$^\#$\\ \\
$^*$ Technion---Israel Institute of Technology\\
  \texttt{reshefm@ie.technion.ac.il}\\
$^\#$ Harvard University\\
 \texttt{parkes@eecs.harvard.edu} }
\maketitle

\begin{quote}
\textit{``Sex is the queen of problems in evolutionary biology. 
 Perhaps no other natural phenomenon has aroused so much interest; 
certainly none has sowed so much confusion.''} 
\flushright ---Graham Bell, 1982
\end{quote}

\begin{abstract}
We consider a recent innovative theory by Chastain et al. on the role of sex in evolution~\cite{PNAS_new}. 
In short, the theory suggests that the evolutionary process of gene recombination implements the celebrated multiplicative weights updates algorithm (MWUA). They prove that the population dynamics induced by sexual reproduction can be precisely modeled by genes that use MWUA as their learning strategy in a particular coordination game. The result holds in the environments of \emph{weak selection}, under the assumption that the population frequencies remain a product distribution.
	
We revisit the theory, eliminating both the requirement of weak selection and any assumption on the distribution of the population. 
Removing the assumption of product distributions is crucial, since as we show, this assumption is inconsistent with the population dynamics.
We show that the marginal allele distributions induced by the population dynamics precisely match the marginals induced by a multiplicative weights update algorithm in this general setting, thereby affirming and substantially generalizing these
earlier results.  

We further revise the implications for convergence and utility or
fitness guarantees in coordination games. In contrast to the claim of Chastain et al.~\shortcite{PNAS_new}, we conclude that the sexual evolutionary dynamics does not entail any property of the population distribution, beyond those already implied by convergence.
\end{abstract}
%

%\category{I.2.11} {Distributed Artificial
%Intelligence}{Multiagent systems}
%%\category{J.4} {Social and Behavioral Sciences}{Economics}
%
%%\terms{Delphi theory}
%\terms{Algorithms; Theory}
%
%\keywords{Evolution; Hedge; Game theory; Convergence}

\section{Introduction}

Connections between the theory of evolution,
machine learning and
games have captured the imagination of researchers for decades.
Evolutionary models inspired a range of applications from genetic algorithms to the design of distributed multi-agent systems~\cite{goldberg1988genetic,cetnarowicz1996application,phelps2008auctions}.
Within game theory, several solution concepts follow evolutionary processes, and some of the most promising dynamics that lead to equilibria in games assume that players learn
the behavior of their opponents~\cite{haigh1975game,valiant2009evolvability}.

%In a recent paper, Chastain et al.~\shortcite{ppad} suggested a new and surprising connection among the three fields, a
%connection that sheds light on one of the most intriguing questions in
%evolution-- ``the problem of sex.''

%A common argument is that sex (i.e., the population dynamics induced
%by sexual reproduction) optimizes some property other than 
%the average
%fitness of the population. The most prominent alternative property
%that has been suggested is
%\emph{diversity}~\cite{gell1995quark,gross1996alternative}, but the
%exact process by which sex contributes to diversity is not completely
%understood. % Another alternative property that was suggested is \emph{evolvability}, which is the ability of separate genes to independently contribute to fitness~\cite{wagner2005robustness}.
%%

%In a series of simulations, they show that sexual reproduction results in
%a transient increase of mixability, and suggest this as a possible explanation for the role of sex in evolution. % increases evolvability.

%Chastain et al.~\shortcite{ppad} take a further step in pushing forward the mixability theory, and en route combine evolution, learning, and games. 

A different connection between sex, evolution and machine learning was recently suggested by Chastain, Livnat, Papadimitriou and Vazirani~\shortcite{PNAS_new}. 
 As they explain, also
referring to Barton and Charlesworth~\shortcite{barton1998sex}, sexual
reproduction is costly for the individual and for the society in terms of time and energy, and often breaks successful gene combinations. From the
perspective of an individual, sex dilutes his or her genes by only
transferring half of them to each offspring. Thus the question that arises is why sexual reproduction is so common in nature, and why is it so successful. Chastain et al.~\shortcite{PNAS_new} suggest that the evolutionary process under sexual reproduction effectively implements a celebrated no-regret learning algorithm.
The structure of their argument is as follows. 
 
First, they restrict attention to a particular class of fitness landscape where \emph{weak selection} holds. Informally, weak selection means that the fitness difference between genotypes 
 is bounded by a small constant, i.e., there are no extremely good or extremely bad gene combinations.\footnote{A gene takes on a particular form, known as {\em allele}. By a {\em genotype}, or gene combination, we refer to a set of alleles-- one for each gene. The genotype determines the properties of the creature, and hence its fitness in a given environment.}
%
%The genotype can be thought of as the instructions that define the  {\em phenotype}  
%
%It is known that under weak selection, the distribution of the population converges to a steady state~\cite{nagylaki1993evolution}. 
Second, they  consider the distribution of each gene's alleles as a \emph{mixed strategy} in a matrix-form game, where there is one player for each gene. The game is an 
\emph{identical interest game}, where each player gets the same utility--- thus the joint distribution of alleles corresponds to the mixed strategy of each player, and the expected payoff of the game  corresponds to the average fitness level of the population.

Chastain et al.~\shortcite{PNAS_new} provide a  correspondence
between the sexual population dynamics  and the {\em multiplicative weights update algorithm} (MWUA)~\cite{littlestone1994weighted,cesa1997use}. %\footnote{In this version of the algorithm, each player observes and reacts to the mixed strategy last played by the other players, rather than the empirical distribution of sampled actions.} % That is, in each generation the distribution of the population (joint strategy) changes as if each gene (player) used the MU algorithm to update its distributions of alleles (its ``strategy'').
In particular, they establish a correspondence between strategies
adopted by players in the game that adopt MWUA  and 
the population dynamics,
under an assumption that the fitness matrix is in the weak selection regime, and that the population dynamic retains
the structure of a product distribution on alleles. % In our view, 
 %both of these assumptions of are quite strong. %and even though it draws some support from 
%their assumption of weak selection and the theory of Nagylaki et al.~\shortcite{nagylaki1999convergence},
%Further, even though  the assumption that population remains in a product distribution  Indeed, we show in the appendix that it is counterfactual
%with their model of the population dynamics under sex and even for \dcp{very small $s$
%in sense of ws, small $r$?}.
%
With this correspondence in place, 
these authors
apply the fact that using MWUA in a repeated game leads to diminishing regret for each player
to conclude that the regret experienced by genes also diminishes.
They   interpret this result as maximization of a property of the population distribution, namely ``the sum of the expected cumulative differential fitness over the alleles, plus the distribution's entropy.''

%\footnote{In a previous version of the paper~\cite{ppad}, the authors used the term ``mixability'' to describe the property maximized by the process. However this new definition of mixability was inconsistent with the one from~\cite{PNAS}, and the term does not appear in the new version.} %\footnote{See penultimate section for a discussion on mixability.}  

We believe that such a multiagent abstraction of the evolutionary process can contribute much to our understanding, both of evolution and of learning algorithms. Interestingly, the agents in this model are not the creatures in the population, nor Dawkins'~\cite{dawkins2006selfish} genetic properties (alleles) that compete one another, but rather genes that share a mutual cause.

%\newpage
\subsec{Our Contribution}

We show that the main results of Chastain et al.~\shortcite{PNAS_new}
can be substantially generalized.
Specifically, we consider the two standard population dynamics (where recombination acts before selection (RS), and vice versa (SR)), and show that each of them \emph{precisely} describes the marginal allele distribution under a variation of the multiplicative updates algorithm 
that is described for correlated strategies. This
correspondence holds for any number of genes/players, any fitness matrix,  any recombination rate, and any initial population frequencies. In particular, and in contrast to Chastain et al., we do not assume weak selection or require the population distribution 
remains a product distribution (i.e., with allele probabilities that are independent),
and we allow both the SR model and the RS model.\dcp{explain that even though their paper introduces the SR model we have clarified through personal comm that they intend the RS model, and their results only hold in this model} \rmr{i don't think this is the place. Any short sentence will be very inaccurate, and we explain this later.}
%
%\dcp{prefer not to specify parameter-free variation here to not
%make the result seem unnecessarily weak}
%, or any other assumption, applies. 
%We provide a self-contained
%proof, which does not require a definition of mixability. 
%

\dcp{perhaps begin this para by reminding reader why connection with MU should be of interest}
We discuss some of the implications of this correspondence between these biological and algorithmic processes for theoretical convergence properties.  Under weak selection, the observation that the cumulative regret of every gene is bounded follows immediately from known convergence results, both in population dynamics and in game theory (see related work).
 We show that under the SR dynamics, every gene still has a bounded cumulative regret, without assuming weak selection or a product distribution.
%Then, we show a tight connection can be drawn between  results from population genetics under weak selection, and convergence of the MU algorithm in games. Specifically, in the weak selection case studied in \cite{ppad}, the dynamics converges to point distribution, which follows  as a special case of a classic result on population dynamics under weak selection~\cite{nagylaki1993evolution,nagylaki1999convergence}. We show that it also follows from the convergence of a specific form of the MU algorithm in congestion games (a recent result by  Kleinberg et al.~\shortcite{kleinberg2009multiplicative}), and that the two results are in a sense equivalent. 

Our analysis also uncovers what we view as 
one technical gap and one conceptual gap regarding the fine details in the original argument of Chastain et al.~\shortcite{PNAS_new}. We believe that due to the far reaching consequences of the theory it is important to rectify these details.  First, according to the population dynamics the population frequencies may become highly correlated (even under weak selection), and thus it is important to avoid the assumption on product distributions. 
Second, the property that is supposedly maximized by the population dynamics is already entailed by the convergence of the process (regardless of what equilibrium is reached). We should therefore be careful when  interpreting it as some nontrivial advantage of the evolutionary process. 
%

%We conclude with some additional critique on the conceptual claims in Chastain et al.~\shortcite{ppad}, in particular about the role of mixability. \dcp{bit inconsistent with FN 2 on previous page which says we won't get into mixability. maybe drop that footnote?}%and place their and our work in the context of known results from evolutionary biology and game theory. In particular we show that the definition in \cite{ppad} of mixability differs from the one in \cite{PNAS} but coincides with fitness; and that even under weak selection the multiplicative weights algorithm does not provide any guarantee on mixability or fitness, other than convergence.

\subsec{Related Work}
%\paragraph{Convergence of population dynamics}

The multiplicative weights update algorithm  (MWUA) is a general name for a
broad class of methods in which a decision maker facing uncertainty
(or a player in a repeated game), updates her strategy. While
specifics vary, all variations of MWUA increase the probability of
actions that have been more successful in previous rounds.
%
%
%\newpar{Convergence of MU algorithms}
In general games, the MWUA dynamic is known to lead to diminishing
regret over time~\cite{jafari2001no,blum2007learning,kale2007efficient,cesa2007improved}, but
does not, in general, converge to a Nash equilibrium of the game.
For some classes of games better convergence results are known; see Section~\ref{sec:converge_weak} for details.
%Identical interest games are a special case of \emph{potential
  %games}~\cite{monderer1996potential}, where the potential is exactly
%the utility. %However, most research on repeated identical interest
%%games has focused on dynamics that converge to the \emph{optimal
  %%profile}, i.e., the joint action yielding the highest utility~\cite{pazgal1997satisficing,kim1999satisficing,marden2009payoff}, and is less relevant to us.  
%Kleinberg et
%al.~\shortcite{kleinberg2009multiplicative} analyze the dynamics of
%the {\em Hedge} algorithm (another variation of MU) in congestion
%games.  They prove that the $\eps$-Hedge algorithm almost always
%converges to a pure Nash equilibrium% of the subgame defined by the support
%%of the initial strategies
%, for a sufficiently small $\eps$. %They also study when convergence to a \emph{pure} Nash equilibrium occurs.
%\dcp{sentence to say why relevant} \rmr{this should be clear now from the Our Contribution section}

%\newpar{Population dynamics}
The \emph{fundamental theorem of natural selection}, which dates back
to Fisher~\shortcite{Fisher30}, states that the population dynamics of
a single-locus diploid always increases the average fitness of each
generation, until it reaches
convergence~\cite{MS59%,SM59
,li1969increment}.\footnote{Roughly, a single-locus means there is only one property that determines fitness, for example \emph{eye color} or \emph{length of tail}. Multiple loci mean that fitness is determined by a combination of several such properties. We explain what are diploids and haploids in the next section.} The fundamental theorem further relates the
rate of increase to the variance of fitness in the  population.
In the general case, for genotypes with more than a single locus, the
fundamental theorem does not hold, although constructing a counter
example where a cycle occurs is
non-trivial~\cite{hastings1981stable,hofbauer1984hopf}.
However, convergence of the population dynamics  has been
shown to hold when the fitness landscape has some specific properties,
such as \emph{weak selection}, or \emph{weak epistasis}~\cite{nagylaki1999convergence}.\footnote{Weak epistasis means that the
various genes have separate, nearly-additive contribution to
fitness. It is incomparable to
weak selection.}
%\dcp{is weak epistasis incomparable to weak selection, or
%stronger?}
%\rmr{incomparable. it means that $W$ has rank 1.}

%\dcp{intro sentence to say what asex dynamics are} 
In \emph{asexual} evolutionary dynamics, every descendent is an exact copy of a
single parent, with more fit parents producing more offspring (``survival of the fittest'').
Regardless of the number of loci, asexual dynamics coincides with MWUA by a single 
player~\cite{borgers1997learning,hopkins1999learning}.  
 %The average fitness of the actual population.  
Chastain et al.~\shortcite{PNAS_new} were the first to suggest that a
similar correspondence can be established for sexual population
dynamics.

A different connection between sexual reproduction and regret minimization was later suggested by \citet{edhan2017sex}. Their argument hinges on a crucial modeling assumption where the population is \emph{finite}, and the advantage of sex is realized via better sampling. 
%Recently, Livnat et al.~\shortcite{PNAS} proposed a theory of {\em
  %mixability} in regard to the role of sex in evolution. Mixability,
%according to Livnat et al., measures how well members of the
%population mix with others (in terms of fitness of their potential
%offsprings), rather than

%A corollary of this result and of Proposition~\ref{th:sex_PW} is that Conjecture~\ref{conj:converge} holds in the limit for product distributions under weak selection (which is known).
  
%However, \cite{kleinberg2009multiplicative} show a stronger result: that convergence is to a subset of Nash equilibria, that coincide with pure Nash equilibria in all but zero-measure fraction of games. Thus they essentially prove that Conjecture~\ref{conj:pure} holds under weak selection as well. 

%We are unaware of papers that study the convergence properties of the particular variation of the MU algorithm used here and in \cite{kale2007efficient,ppad} (i.e., of the parameter-free PW algorithm). 

\section{Definitions}

\begin{table*}[t]
	\begin{center}
	\begin{tabular}{||l|c|c||c|l|c|c|cc|l|c|c|c}
	\cline{1-3}\cline{5-7}\cline{10-12}
	\cline{1-3}
	 $w_{ij}$ & $b_1$ & $b_2$ & ~~~~ & $p^0_{ij}$ & $b_1$& $b_2$ & $\vec y^0$ & & $p^{1}_{ij}$ & $b_1$ & $b_2$ & $\vec y^0$\\
	\cline{1-3}\cline{5-7}\cline{10-12}
 $a_1$  &  1  &  0.5 &   & $a_1$ &  0.1 & 0.15 & 0.25 & & $a_1$   & 0.094 & 0.080 & 0.174\\
\cline{1-3}\cline{5-7}\cline{10-12}
$a_2$  &  1.5  &  1.2  &   & $a_2$ & 0.1 & 0.15 & 0.25 & &$a_2$  & 0.158 & 0.170 & 0.328 \\
\cline{1-3}\cline{5-7}\cline{10-12}
$a_3$  & 1.3 & 0.8 & & $a_3$ & 0.2 & 0.3 & 0.5 & &$a_3$  & 0.255 &  0.242 & 0.497  \\
\cline{1-3}
	\cline{1-3}\cline{5-7}\cline{10-12}
\multicolumn{4}{c}{}& \multicolumn{1}{c}{$\vec x^0$} & \multicolumn{1}{c}{0.4} & \multicolumn{1}{c}{0.6} & \multicolumn{2}{c}{} &\multicolumn{1}{c}{$\vec x^1$} & \multicolumn{1}{c}{0.5075} & \multicolumn{1}{c}{0.4925}\\
\end{tabular}
	\caption{\label{tab:example}  An example of a 2-locus haploid fitness matrix,  with $n=3,m=2$.  There are 6 allele combinations, or genotypes. On the left we give the fitness of each combination $\tup{a_i,b_j}$. 
	In the middle, we provide an initial population distribution, and on the right we provide the distribution after one update step of the SR dynamics, for $r=0.5$.}
	\end{center}
	\end{table*}

We follow the definitions of Chastain et al.~\shortcite{PNAS_new} 
where
possible. For more detailed explanation of the biological terms and equations,
see B\"{u}rger~\shortcite{burger2011some}.
% %We write down the biological terms where relevant, but this terminology is not required to understand the results.%

\subsec{Population dynamics}

A \emph{haploid} is a creature that has only one copy of each gene.
%, which is placed in a particular locus (``place'' in Latin).
Each gene has several distinct alleles. For example, a gene for eye
color can have alleles for black, brown, green or blue color. In
contrast, people are \emph{diploids}, 
and have \emph{two copies} of each
gene, one from each parent.

Under asexual reproduction, an offspring inherits all of its genes
from its single parent. In 
the case of sexual reproduction, each parent
transfers half of its genes. Thus a haploid inherits half of its properties from one parent and half from the other parent.
% from each parent.
%
To keep the presentation simple we focus on 
the case of a {\em 2-locus haploid}. This means that there are two
genes, denoted $A$ and $B$. In the appendix we extend the definitions and results to $k$-locus haploids, for $k>2$. 
Gene $A$ has $n$ possible alleles
$a_1,\ldots,a_n$, and gene $B$ has $m$ possible alleles $b_1,\ldots,b_m$.
 It is possible that
$n\neq m$.  We denote the set $\{1,\ldots,n\}$ by $[n]$. A pair $\tup{a_i,b_j}$ of alleles defines a \emph{genotype}.

Let $W=(w_{ij})_{i\leq n,j\leq m}$ denote a \emph{fitness matrix}.  The
fitness value $w_{ij}\in \mathbb R_+$ can be interpreted as the expected number of
offspring of a creature whose genotype is $\tup{a_i,b_j}$.
 We assume
that the fitness matrix is fixed, and does not change throughout
evolution. See Table~\ref{tab:example} for an example.

We denote by $P^t=(p_{ij}^t)_{i\leq n,j\leq m}$ the distribution of
the population at time $t$. The \emph{average fitness} $\ol w^t$ at
time $t$ is written as $\ol w^t = \ol w(P^t) = \sum_{ij}p^t_{ij}w_{ij}$.

 For
example, the populations in Table~\ref{tab:example} have an average
fitness of $\ol w(P^0) = 1.005$  and $\ol w(P^1)=1.1012$.
Denote by $x^t_i=\sum_j p^t_{ij}$ and $y^t_j = \sum_i p^t_{ij}$ the
{\em marginal frequencies} at time $t$ of alleles $a_i$ and $b_j$,
respectively. 
Clearly $p^t_{ij}=x^t_i y^t_j$ for all $i,j$ iff $P^t$ is a
product distribution. In the context of population dynamics, the set
of product distributions is also called the \emph{Wright manifold}.
For general distributions, $D^t_{ij} = p^t_{ij}-x^t_i y^t_j$ is called the
\emph{linkage disequilibrium}.

%\newpar{Weak selection.}
 
The \emph{selection strength} of $W$ is the minimal $s$ s.t. $w_{ij}\in[1-s,1+s]$ for all $i,j$. %Set $\Delta =(W-1)/s$ 
%(e.g., in Table~\ref{tab:example} we have  $s=1.5-1 = 0.5$).  
We say that $W$ is in the \emph{weak selection} regime if
$s$ is very small, i.e., all of $w_{ij}$ are close to $1$.
%\dcp{bit uncomfortable with this. does it mean to specify WS as studied under
%the limit as $s$ goes to 0?}
%\rmr{No. Formal results should have $s$ as a parameter (which is not how it is used in \cite{ppad}). We do not make formal claims about what happens under weak selection, so perhaps we can give a less formal definition.}
%
%\paragraph{Mixability}
%Following \cite{PNAS}, we denote by $M_i=\sum_j w_{ij}$ (or $M_j$) the mixability of allele $a_i$ ($b_j$). Crucially, $M_i$ is an unweighted average of fitness values, that does not depend on time or distribution. 
%Following \cite{ppad}, we denote by 
%$m^t_x(i)=\sum_j y^t_{j} \Delta_{ij}$ the mixability of allele $a_i$ at time $t$. Note that $m^t_x,m^t_y$ are weighted by the current distribution $P^t$.  We will call $M$ \emph{static mixability} and $m^t$ \emph{dynamic mixability}. 

\newpar{Update step}

In {\em asexual reproduction}, every creature of genotype
$\tup{a_i,b_j}$ has in expectation $w_{ij}$ offspring, all of which
are of genotype $\tup{a_i,b_j}$. Thus there is  \emph{selection}
w/o \emph{recombination}. Let $p^{(S)}_{ij} = w_{ij}p^t_{ij}$. After normalization, the frequencies in the next period are
$p^{t+1}_{ij} = \frac{1}{ \ol w^t} p^{(S)}_{ij} = \frac{1}{ \ol w^t}w_{ij}p^t_{ij}$.

In {\em sexual reproduction}, every pair of creatures, say of
genotypes $\tup{a_i,b_l}$ and $\tup{a_k,b_j}$ bring offspring who may
belong to genotype
$\tup{a_i,b_l}$, $\tup{a_k,b_j}$, $\tup{a_i,b_j}$, or $\tup{a_k,b_l}$. Thus,
in the next generation, a creature of genotype $\tup{a_i,b_j}$ can be
the result of combining one parent of genotype $\tup{a_i,?}$ with
another parent of genotype $\tup{?,b_j}$. There are two ways to infer the distribution of the next generation, depending on whether recombination occurs before selection or vice versa (see, e.g., \cite{michalakis1996interaction}). We describe
each of these two ways next.
%\ccw{Recombination-rate convention (\#3): the sentence above (offspring equally likely to be any of the four genotypes) describes uniform crossover, where half the offspring are copies of a parent. But Eq.~\eqref{eq:sex_SR_r} uses $p^{SR}$, in which every recombinant offspring of $\tup{a_i,b_l}\times\tup{a_k,b_j}$ is $\tup{a_i,b_j}$, so $r$ there is the probability that the offspring is recombinant. Read literally, the sentence gives rate $r/2$. Suggest: ``with probability $r$ the offspring is $\tup{a_i,b_j}$ or $\tup{a_k,b_l}$ (equally likely), and otherwise a copy of one of its parents.'' The propositions hold for every $r$; only the numerical examples depend on the convention.}

\newpar{Selection before recombination (SR)}
 Summing over all possible
matches and their frequencies we get (before normalization):
%
%\vspace{-2mm}
$$
p^{(SR)}_{ij} = \sum_{l\in [m]}\sum_{k \in [n]}
    p^t_{il} w_{il} p^t_{kj} w_{kj}.$$

Note that to normalize we should divide this term by $(\ol w^t)^2$.

In addition, the 
\emph{recombination rate}, $r\in[0,0.5]$, determines the
part of the genome that is being replaced in {\em crossover}, so $r=0.5$ means that the entire genome is the result of recombination, whereas $r=0$ means no recombination occurs, and the offspring are genetically identical to one of their parents. Population frequencies in the next period after normalization are thus set as:
\labeq{sex_SR_r}
{p^{t+1~(SR)}_{ij} = r \frac{1}{(\ol w^t)^2} p^{(SR)}_{ij} + (1-r)\frac{1}{\ol w^t}p^{(S)}_{ij}.} 
 %The value of $r$ in actual biological environments in the range $[0,0.5]$.%\footnote{Chastain et al.~\shortcite{ppad} assume $r=1$ (see p.), but do not in fact use this assumption.} %We will do the same in parts of our paper. %, and thus $p^{t+1}_{ij} = p^{(S)}_{ij}$.

\newpar{Recombination before selection (RS)}
With only recombination, the frequency of the genotype $\tup{a_i,b_j}$ is the product of the respective probabilities in the previous generation, i.e., $p^{(R)}_{ij} =  x^t_i y^t_j$.
When recombination occurs before selection, we  have (before normalization):
$$p^{(RS)}_{ij} = w_{ij}p^{(R)}_{ij} = w_{ij} x^t_i y^t_j.$$

Taking into account the recombination rate and normalization, we get,
\labeq{sex_RS_r}
{p^{t+1~(RS)}_{ij} = \frac{1}{\ol w^{(R)}} ( r p^{(RS)}_{ij} +(1- r) p^{(S)}_{ij}) =  \frac{1}{\ol w^{(R)}} w_{ij}(p^t_{ij} - rD^t_{ij}),}
where $\ol w^{(R)}=\sum_{ij}w_{ij}(p^t_{ij} - rD^t_{ij})$, and $D^t_{ij}$ is the linkage disequilibrium at time $t$.

\medskip
When the dynamics is clear form context we omit its superscript and write $P^{t+1}$ or $p_{ij}^{t+1}$. 
For an example of change in population
frequencies, see Tables~\ref{tab:example}-\ref{tab:example_P_r2}.
We say that $P^t$ is a \emph{stable state} under a particular dynamics, if $P^{t+1} = P^t$.

\subsec{Identical interest games}

An identical interest game of two players is defined by a payoff
matrix $G$, where $g_{ij}$ is the payoff of each player if the first
plays action $i$, and the second plays action $j$. A mixed strategy of
a player is an independent distribution over her actions.  The mixed
strategies $\vec x,\vec y$ are a \emph{Nash equilibrium} if no player
can switch to a strategy that has a strictly higher expected payoff.
That is, if for any action $i'\in [n]$, $\sum_j y_j g_{i'j} \leq
\sum_{i}\sum_j x_i y_j g_{ij}$, and similarly for any  $j'\in [m]$. \ccw{Typo (\#4): $x_iy_i$ should be $x_iy_j$.}

Every fitness matrix $W$ induces an identical interest game, where
$g_{ij}=w_{ij}$.  This is a game where each of the two genes selects
an allele as an action (or a distribution over alleles, as a mixed
strategy). A matrix of population frequencies $P$ can be thought of as \emph{correlated strategies} for the players. The expected payoff of each player under these strategies is $\ol
w(P)$.  Given a distribution $P$, $G|_P$ is the subgame of
$G$ induced by the support of $P$. That is, the subgame where action
$i$ is allowed iff $p_{ij}>0$ for some $j$. %, and likewise for action $j$.

\subsec{Multiplicative updates algorithms}
\label{sec:MU}

Suppose that two players play a game $G$  (not necessarily identical interest) repeatedly. Each player observes the strategy
of her opponent in each turn, and can change her own strategy
accordingly.\footnote{
We assume that the player observes the full joint distribution $P^t$, and can thus infer the (expected) utility of every action $a_i$ at time $t$.} %\dcp{when kale says this does he mean $P^t$ to mean the play in the last round, or the empirical play. good to clarify, I think}}
%
%A simple reactive algorithm that a player can adopt is the
%\emph{best-response} rule, where the player always selects the optimal
%(pure) action w.r.t. the strategy last played by the opponent.
%%
One prominent approach is to gradually put more weight (i.e.,
probability) on pure actions that were good in the previous steps.
\sloppy Many variations of the multiplicative weights update algorithm (MWUA)  are built
upon this idea, and some have been applied to
strategic settings~\cite{blum2007learning,marden2009payoff,kleinberg2009multiplicative}.
We  follow  the variation used by Chastain et
al.~\shortcite{PNAS_new}. %, which in turn follows Kale~\shortcite{kale2007efficient}.
This variation is equivalent to the {\em Polynomial Weights} (PW) algorithm~\shortcite{cesa2007improved}, under the assumption that the utility of all actions $(a_i)_{i\leq n}$ is observed after each period (see Kale~\shortcite{kale2007efficient}, p.~10).

\begin{table}[t]
	\begin{center}
	{\small
	\begin{tabular}{|l|c|c|cc|l|c|c|c}
	\cline{1-3}\cline{6-8}
	 %$p^{1}_{ij}$
		SR & $b_1$ & $b_2$ & $\vec y^1$ & & %$p^{1}_{ij}$ 
		RS & $b_1$ & $b_2$ & $\vec y^1$ \\
	\cline{1-3}\cline{6-8}
 $a_1$   & 0.088 & 0.086 & 0.174 &  & $a_1$ & 0.099 & 0.075 & 0.174\\
\cline{1-3}\cline{6-8}
$a_2$  & 0.167 & 0.162 & 0.328 & & $a_2$ & 0.149 & 0.179 & 0.328\\
\cline{1-3}\cline{6-8}
$a_3$  & 0.252 &  0.245 & 0.497 & & $a_3$ &0.258 & 0.239  & 0.497\\
	\cline{1-3}\cline{6-8}
	\multicolumn{1}{c}{$\vec x^1$} & \multicolumn{1}{c}{0.507} & \multicolumn{1}{c}{0.493} & \multicolumn{1}{c}{}& \multicolumn{1}{c}{} & \multicolumn{1}{c}{$\vec x^1$} & \multicolumn{1}{c}{0.507} & \multicolumn{1}{c}{0.493} &  \\
\end{tabular}}
	\caption{\label{tab:example_P} Population frequencies $P^1=(p^1_{ij})_{ij}$, i.e. after one update step, for the simplified case of $r=1$ (which is not a biologically valid). On the left, we provide frequencies when selection occurs before recombination, whereas on the right recombination precedes selection. The updated marginal frequencies of alleles are the same in both cases.} %The average fitness changes from $0.9902$ to $0.9819$.}\vspace{-2mm}
	\end{center}
	\end{table}

\begin{table}[t]
	\begin{center}
	{\small
	\begin{tabular}{|l|c|c|cc|l|c|c|c}
	\cline{1-3}\cline{6-8}
	 %$p^{1}_{ij}$  
	SR & $b_1$ & $b_2$ & $\vec y^1$ & & %$p^{1}_{ij}$ 
	RS & $b_1$ & $b_2$ & $\vec y^1$ \\
	\cline{1-3}\cline{6-8}
 $a_1$   & 0.094 & 0.08 & 0.174 &  & $a_1$ & 0.099 & 0.074 & 0.174\\
\cline{1-3}\cline{6-8}
$a_2$  & 0.158 & 0.170 & 0.328 & & $a_2$ & 0.149 & 0.179 & 0.328\\
\cline{1-3}\cline{6-8}
$a_3$  & 0.255 &  0.242 & 0.497 & & $a_3$ &0.258 & 0.239  & 0.497\\
	\cline{1-3}\cline{6-8}
	\multicolumn{1}{c}{$\vec x^1$} & \multicolumn{1}{c}{0.507} & \multicolumn{1}{c}{0.493} & \multicolumn{1}{c}{}& \multicolumn{1}{c}{} & \multicolumn{1}{c}{$\vec x^1$} & \multicolumn{1}{c}{0.507} & \multicolumn{1}{c}{0.493} &  \\
\end{tabular}}
	%\ccw{Typo (\#4): SR $(a_1,b_2)$ should be $0.080$, as in Table~\ref{tab:example} (recomputed: $0.0802$).}
	\caption{\label{tab:example_P_r} Population frequencies $P^1$ for $r=0.5$. Under RS we get the same frequencies as with $r=1$ (and in fact any other value of $r$). This is because when $P^0$ is a product distribution there is  no effect of recombination under RS. }
	\end{center}
	\end{table}

\begin{table}[t]
	\begin{center}
	\begin{small}
	\begin{tabular}{|l|c|c|cc|l|c|c|c}
	\cline{1-3}\cline{6-8}
	 %$p^{2}_{ij}$  
	SR & $b_1$ & $b_2$ & $\vec y^2$ & & %$p^{2}_{ij}$ 
	RS & $b_1$ & $b_2$ & $\vec y^2$ \\
	\cline{1-3}\cline{6-8}
 $a_1$   & 0.079 & 0.042 & 0.122 &  & $a_1$ & 0.09 & 0.034 & 0.124\\
\cline{1-3}\cline{6-8}
$a_2$  & 0.228 & 0.172 & 0.401 & & $a_2$ & 0.203 & 0.195 & 0.398\\
\cline{1-3}\cline{6-8}
$a_3$  & 0.294 &  0.183 & 0.477 & & $a_3$ &0.305 & 0.173  & 0.478\\
	\cline{1-3}\cline{6-8}
	\multicolumn{1}{c}{$\vec x^2$} & \multicolumn{1}{c}{0.602} & \multicolumn{1}{c}{0.398} & \multicolumn{1}{c}{}& \multicolumn{1}{c}{} & \multicolumn{1}{c}{$\vec x^2$} & \multicolumn{1}{c}{0.598} & \multicolumn{1}{c}{0.402} &  \\
\end{tabular}\end{small}
	\caption{\label{tab:example_P_r2} Population frequencies $P^2$, i.e. after two update steps for $r=0.5$. The marginal distributions are no longer the same, since $P^1$ under RS is not a product distribution.}
	\end{center}
	\end{table}

\newpar{Polynomial Weights}
We use the term PW to distinguish this from
other variations of MWUA.  For any $\eps> 0$, the
$\eps$-PW algorithm for a \emph{single} decision maker is defined as follows. Suppose first that in time $t$,
the player uses strategy $\vec x^t$. Let $g^t_i$ be the utility to the player
when playing some pure action $i\in [n]$.
According to the $\eps$-PW algorithm, the strategy of the player in
the next step would be
%
%\labeq{eq:PW}
$x^{t+1}(i) \propto x^t(i) (1+\eps g^t_i)$,
where $\propto$ stands for ``proportional to'' (we need to  normalize, since $\vec x^{t+1}$ has to be a valid distribution). 
%Player~2 updates strategy $\vec y^{t+1}$ in the same way. 
%
A special case of the algorithm is the limit case $\eps\rightarrow
\infty$, where $x^{t+1}(i) \propto x^t(i) g^t_i$;
i.e., the
probability of playing an action increases proportionally to its
expected performance in the previous round. Unless specified
otherwise we assume this limit case, which we refer to as the 
\emph{parameter-free PW}. %\footnote{This algorithm is not related to the parameter-free
  %Hedge algorithm in \cite{chaudhuri2009parameter}.\dcp{Need to say this?} \rmr{maybe not. just don't want anyone to confuse.}\dcp{prob OK to drop, I think}}

\newpar{PW in Games}
The fundamental feature of the PW algorithm is that the probability of playing action $a_i$ changes proportionally to the expected utility of  action $a_i$.

  %attained in the previous round $t$. %When $\vec x^t,\vec y^t$ are independent, this expected utility has a clear interpretation $g^t_i = \sum_j y^t_j g_{ij}$. 
%
Consider 2-player game $G$, where $g_{ij}$ is the utility (of both players if $G$ is an identical interest game) from the joint action $\tup{a_i,b_j}$. In the context of a game, we can think of at least two different interpretations of the utility of playing $a_i$, derived from the joint distribution $P$.
For this, let $y^t_j(a_i) = P^t(b_j | a_i)$, i.e., the probability that player~2 plays $b_j$ \emph{given that} player~1 plays $a_i$, according to the distribution $P^t$.  %Let $\ol y^t_j = \sum_{i'}y^t_j(a_{i'})$, i.e. the marginal probability of playing $b_j$. 
The two interpretations we have in mind are:
\begin{itemize}
	\item Set $\ul g^t_i = \sum_j y^t_j(a_i) g_{ij}$. This is the expected utility that player~1 would get for playing $a_i$ in round $t$. This definition is consistent with common interpretation of expected utility in games (e.g., in Kale~\shortcite{kale2007efficient}, Sec.~2.3.1).
\item Set $\ol g^t_i = \sum_j y^t_j g_{ij}$. This is the expected utility that player~1 will get in the next round for playing $a_i$ if player~2 will select an action independently according to her current marginal probabilities. 
Thus each agent updates her strategy \emph{as if} the strategies are independent, and ignoring any observed correlation. 
This definition 
results in the PW algorithm  used in Chastain et al.~\shortcite{PNAS_new}.
\end{itemize}

%\footnote{By using this measure, the player updates her strategy as if the other player is guaranteed to remain correlated in the same way. This interpretation is quite similar to the more traditional approach: the empirical utility that player~1 experienced in all previous round when $a_i$ was played is a proxy of $\ul g^t_i$. }

\medskip 
The above definitions require some discussion.
While the traditional assumption is that each player only observes a \emph{sample} from the joint distribution at each round, and updates the strategy based on the empirical distribution, strategy updates can also be performed in the same way when the player observes the joint distribution at round $t$, even if it is hard to imagine such a case occurs in practice. 

\if 0
Intuitively, under the first interpretation of expected utility the player considers correlation (as reflected for example in the empirical joint distribution), whereas under the second interpretation the player assumes independence, and then uses the marginals to compute the expected utility.
%\footnote{We can also think of the two approaches as the two sides of Newcomb's paradox~\cite{nozick1969newcomb}: the player observes a correlation, even though deciding on a strategy cannot change the expected utility of each action. Thus it is not obvious whether the observed correlation should be considered in the strategy update.} 
Clearly, when $P^t$ is a product distribution (as in Chastain et al.~\shortcite{PNAS_new}) then $\ol g^t_i = \ul g^t_i $, and the algorithms coincide.

\fi

Intuitively, under the first interpretation, the player considers correlation, whereas under the second the player assumes independence, and then uses the marginals to compute the expected utility.
E.g suppose that players play \textsc{Rock-Paper-Scissors}, and the history is 100 repetitions of the sequence [(\textsc{Ro},\textsc{Pa}) (\textsc{Pa},\textsc{Sc}) (\textsc{Sc},\textsc{Sc})]. Then under the first interpretation the best action for agent~1 is \textsc{Sc} (since it leads to the best expected utility); whereas under the second interpretation the best action for agent~1 is \textsc{Ro}, since agent~2 is more likely to play \textsc{Sc}.\footnote{We can also think of the two approaches as the two sides of Newcomb's paradox~\cite{nozick1969newcomb}: the player observes a correlation, even though deciding on a strategy cannot change the expected utility of each action. Thus it is not obvious whether the observed correlation should be considered in the strategy update.}   
Clearly, when $P^t$ is a product distribution (as in Chastain et al.~\shortcite{PNAS_new}) then $\ol g^t_i = \ul g^t_i $, and the algorithms coincide.
We can also combine the two interpretations to induce new algorithms. We thus define the \PW{\alpha} algorithm (either Parameter-free \PW{\alpha} or $\eps$-\PW{\alpha}), where the probability of playing $a_i$ is updated according to $g^{t,(\alpha)}_i = \alpha \ol g^t_i + (1-\alpha) \ul g^t_i$.

\newpar{Exponential Weights}

The \emph{Hedge} algorithm~\cite{freund1995desicion,freund1999adaptive} is another 
variation of
MWUA that is very similar to PW. The difference is that the weight of
action $i$ in each step changes by a factor that is exponential in the
utility, rather than linear. That is, $x^{t+1}(i) \cong x^t(i)
(1+\eps)^{g^t_i}$. For negligible $\eps>0$,
$\eps$-Hedge and $\eps$-PW are essentially the same, but for large
$\eps$ they may behave quite differently.
	
%\section{Sex and Multiplicative Updates}

%This is the main section of our paper, where 
%In this section, 
%we show how to rectify
%and substantially strengthen the results of Chastain et
%al.~\shortcite{ppad}.

\section{Analysis of the SR dynamics}

In this section we prove that the SR population dynamics of marginal
allele frequencies coincide precisely with the multiplicative updates
dynamics in the corresponding game.  This extends Theorem~4 in
Chastain et al.~\shortcite{PNAS_new} (SI text), in that it holds without weak selection or the assumption
of product distributions through multiple iterations. We also generalize
the proposition to hold for any number of loci/players in
Appendix~\ref{apx:loci}.%\dcp{would be awesome if we can claim their
%result as a special case in some kind of special setting where their
%assumption comes for free. i think you have this. add a comment here?} %\rmr{It only holds for $r=1$ under SR with two players, and we decided to play this down. Otherwise their result is not a special case because it is not }

\begin{proposition}
\label{th:SR_PW}
Let $W$ be any fitness matrix, and consider the game $G$ where $g_{ij}=w_{ij}$. Then under the SR population dynamics, for any distribution $P^t$ and any $r\in[0,1]$,
we have
$x^{t+1}_i = \frac{1}{\ol w^t} x^t_i \ul g^t_i$.
\end{proposition}
\begin{proof}\vspace{-0mm} 
%\begin{lemma}
%\label{th:sex_PW_l}
%Let $W$ be any fitness matrix, then under the SR population dynamics, 
%We first show that 
%$x^{t+1}_i = \frac{1}{\ol w^t}\sum_l  p^t_{il} w_{il}$.
%%\end{lemma}
By the SR population dynamics (Eq.~\eqref{eq:sex_SR_r}),
\begin{align*}
x^{t+1}_i &= \sum_j p^{t+1}_{ij} = \sum_j (r \frac{1}{(\ol w^t)^2} p^{(SR)}_{ij} +(1-r)\frac{1}{\ol w^t}p^{(S)}_{ij}) \\
&= r\sum_j \frac{\sum_l\sum_k p^t_{il} w_{il} p^t_{kj} w_{kj}}{(\ol w^t)^2} + (1-r)\sum_j \frac{p^t_{ij}w_{ij}}{\ol w^t}\\
& = r\frac{1}{\ol w^t}\sum_l  p^t_{il} w_{il} \frac{1}{\ol w^t}\sum_k \sum_j p^t_{kj} w_{kj} + (1-r)\sum_j \frac{p^t_{ij}w_{ij}}{\ol w^t} 
\end{align*}
Then since $\sum_k \sum_j p^t_{kj} w_{kj}$ is the average fitness at time $t$,
\begin{align*}
x^{t+1}_i& = r\frac{1}{\ol w^t}\sum_j  p^t_{ij} w_{ij} \frac{1}{\ol w^t} \ol w^t + (1-r)\sum_j \frac{p^t_{ij}w_{ij}}{\ol w^t}\\
& =  r\frac{1}{\ol w^t}\sum_j  p^t_{ij} w_{ij}  + (1-r)\sum_j \frac{p^t_{ij}w_{ij}}{\ol w^t}\\
& = \frac{1}{\ol w^t}\sum_j  p^t_{ij} w_{ij} = \frac{1}{\ol w^t}\sum_j  x^t_{i} y^t_j(a_i) w_{ij}, 
%\qedhere
\end{align*}
thus the recombination factor $r$ does not play a direct role in the new marginal under the SR dynamics. It does have an indirect role though, since it affects the correlation, and thus the marginal distribution at the next generation $t+2$. 
\end{proof}

Theorem~4 in Chastain et al.~\shortcite{PNAS_new} follows as a special case when $P^t$ is a product distribution. Chastain et al. argue that due to a result by Nagylaki~\shortcite{nagylaki1993evolution}, under weak selection the distribution $P^t$ is guaranteed to remain close to the Wright manifold, i.e. that it is approximately a product distribution at any time $t$. Then, by applying their theorem repeatedly, they conclude that the evolutionary dynamics precisely coincides with the dynamics induced by MWUA. However Nagylaki requires certain conditions for the mentioned result to hold, and in particular that the selection strength will be upper bounded by the recombination rate. Recombination rate and Nagylaki's condition are not discussed by Chastain et al. (and we show in the appendix that without a lower bound on $r$,  the population distribution $P^t$ may become very far from the Wright manifold under the evolutionary dynamics). Therefore we argue that one cannot repeatedly apply Theorem~4 of Chastain et al., as it requires $P^t$ to be a product distribution.

 In contrast, Proposition~\ref{th:SR_PW} can be applied repeatedly since there are no assumptions on $P^t$. This gives us the following result, which
holds for any value of $r$.
\begin{corollary}
\label{th:dynamics_cor_SR}
Let $W$ be a fitness matrix, $P^0$ be any distribution. Suppose that $P^{t+1}$ is attained from $P^t$ by the SR population dynamics, and that $\vec x^{t+1}, \vec y^{t+1}$ are attained from $P^t$ by players using  the parameter-free \PW{0} algorithm in the game $G=W$. Then for all $t>0$ and any $i$,  $x^t_i= \sum_{j}p^t_{ij}$. 
\end{corollary}

Two important notes:
\begin{itemize}
    \item the marginal distributions $\vec x^t,\vec y^t$ do not determine $P^t$ completely. Thus the PW algorithm specifies the strategy of each player (regardless of $r$), but not how these strategies are correlated. %\dcp{would prefer to put this caveat about defn of PW earlier in paper as well}
    \item The corollary \emph{does not} compare a sequence of PW steps to the SR dynamics. Indeed, we show in the appendix that even when applied to the same initial distribution, these two dynamics may diverge. 
\end{itemize}

% %Interestingly, in the special case of two genes and $r=1$, it also holds that $P^t$ is a product distribution for all $t>0$, for \emph{any} initial distribution $P^0$ (see Appendix~\ref{apx:wright}). While this is not a biologically plausible assumption, it demonstrates a particularly clean case, where both genes correspond to players that play independently.  Note that in this special case, the marginals $\vec x^t$ and $\vec y^t$ determine $P^t$ completely.\dcp{this is what i propose to put as a FN at the start of this section}\rmr{I think here is good since it relates to the previous comment}

\section{Analysis of the RS dynamics}

Turning to the RS population dynamics, 
our starting point is Lemma~3 in Chastain et al.~\shortcite{PNAS_new} (SI text),
which states that $p^{t+1}_{ij} = \frac{1}{\ol w^{(R)}} w_{ij} x^t_{i} y^t_j$ (under the assumption that $P^t$ is a product distribution). We establish a similar property  for general distributions.
We use the fact that for any fitness matrix $W$ and distribution $P^t$, 
$$p^{t+1}_{ij} = \frac{1}{\ol w^{(R)}} (r  w_{ij} x^t_{i} y^t_{j} + (1-r) w_{ij} p^t_{ij}).$$

This follows immediately from the definition (Eq.~\eqref{eq:sex_RS_r}).
%Note that for the special case when $P^t$ happens to be a  product distribution we get Claim~A.
%
Recall that $g^{t,(r)}_i = \left( r \overline{g}^t_{i} + (1-r) \underline{g}^t_{i} \right)
$. 
We  derive an alternative extension of Theorem~4 in Chastain et al.~\shortcite{PNAS_new} (SI text) for the RS dynamics.
\begin{proposition} Let $W$ be any fitness matrix, and consider the game $G$ where $g_{ij}=w_{ij}$. Then under the RS population dynamics, for any distribution $P^t$ and any $r\in[0,1]$,
$x^{t+1}_{i} = \frac{1}{\ol w^{(R)}} x^t_i g^{t,(r)}_i$.
\label{th:RS_PW}
\end{proposition}
\begin{proof}
By the RS population dynamics (Eq.~\eqref{eq:sex_RS_r}),
\begin{align*}
%\vspace{-1mm}
x^{t+1}_{i} &= \sum_{j}p^{t+1}_{ij} = \sum_{j} \frac{1}{\ol w^{(R)}} (r  w_{ij} x^t_{i} y^t_{j} + (1-r) w_{ij} p^t_{ij})\\ 
&= \frac{1}{\ol w^{(R)}}\sum_{j} (r  w_{ij} x^t_{i} y^t_{j} + (1-r) w_{ij} x^t_{i}y^t_j(a_i))\\ 
&= \frac{1}{\ol w^{(R)}} x^t_i ( r \sum_j w_{ij}  y^t_{j} + (1-r) \sum_j w_{ij} y^t_{j}(a_i)).\\
\end{align*}
Finally, by the definitions of $\ol g^t_{i}$ and $\ul g^t_{i}$,
  $$x^{t+1}_{i} = \frac{1}{\ol w^{(R)}} x^t_i  \left(r \ol g^t_{i} + (1-r) \ul g^t_{i}\right) = \frac{1}{\ol w^{(R)}} x^t_i g^{t,(r)}_i.$$
%\qedhere

In contrast to the SR dynamics, here $r$ appears explicitly in the marginal distribution $\vec x^{t+1}$. 
\end{proof}

So we get that under RS the marginal frequency of allele $a_i$ is updated according to an expected utility that takes only part of the correlation into account. This part is proportional to the recombination rate $r$. We get a similar result to Corollary~\ref{th:dynamics_cor_SR}:

\begin{corollary}
\label{th:dynamics_cor_RS}
Let $W$ be a fitness matrix, $P^0$ be any distribution. Suppose that $P^{t+1}$ is attained from $P^t$ by the RS population dynamics, and that $\vec x^{t+1}, \vec y^{t+1}$ are attained from $P^t$ by players using  the parameter-free \PW{r} algorithm in the game $G=W$. Then for all $t>0$ and any $i$,  $x^t_i= \sum_{j}p^t_{ij}$. 
\end{corollary}

\section{Convergence and Equilibrium}
\label{sec:convergence}

In this section, we consider implications of the general theory on the
correspondence between sexual population dynamics and multiplicative-weights algorithms
on convergence properties.

\subsec{Diminishing regret}
In Chastain et al.~\shortcite{PNAS_new} (Sections~3 and 4 of the SI text), the authors apply standard properties of MWUA  to show that the \emph{cumulative external regret} of each gene is bounded (Corollary~5 there). In other words, if in retrospect gene~1 would have ``played'' some fixed allele $a_i$ throughout the game, the cumulative fitness (summing over all iterations) would not have been much better. This result leans only on  the properties of the  algorithm, and does not require the independence of strategies. Thus we will write a similar regret bound  explicitly in our more general model.

Consider any fitness matrix $W$ whose selection strength is $s$.
Let $AF^T_i = \frac{1}{T}\sum_{t=1}^T \ul g^t_i$ (the average fitness in retrospect if allele~$a_i$ had been used throughout the game), and $AF_{SR}^T = \frac{1}{T}\sum_{t=1}^T \ol w^t$ (the actual average fitness under the SR dynamics). 

\begin{corollary}\label{th:regret}
For any $T\in \mathbb N$, any  $s\in(0,\frac12)$ and all $i\leq n$, 
$AF_{SR}^T \geq AF^T_i - s^2 - \ln(n)/T$.
\end{corollary}
\begin{proof}
Set $\Delta_{ij} = \frac{w_{ij}-1}{s}$;  $m^{(t)}_i =\sum_j y^t_j(a_i)\Delta_{ij}$, and $\eps=s$. Note that $\Delta_{ij}$ and   $m^{(t)}_i$ are in the range $[-1,1]$. Intuitively, $m^{(t)}_i$ is the expected profit of player~1 from playing action $a_i$ in the ``differential game'' $\frac{W-1}{s}$. 

Theorem~3~\cite{kale2007efficient} states that under the $\eps$-PW algorithm\footnote{Kale~\shortcite{kale2007efficient} analyzes the Exponential Weights  algorithm but a slight modification of the analysis works for PW.} 
\labeq{from_kale}
{\sum_{t=1}^T \vec x^{(t)} \vec m^{(t)}  \geq \sum_{t=1}^T m_i^{(t)} - \eps\sum_{t=1}^T ( m_i^{(t)})^2 - \frac{\ln(n)}{\eps},}
where $x^{(t)}_i$ is the probability that the decision maker chose action $a_i$ in iteration $t$ (thus $x^{(t)}_i=x^t_i$ by our notation).

The proof follows directly from the theorem. Observe that  $m^{(t)}_i =(\ul g^t_i -1)/s$, and 
\begin{align*}
\vec x^{(t)} \cdot  \vec m^{(t)} &= \sum_i x^t_i m^{(t)}_i  = \sum_i x^t_i \sum_j y^t_j(a_i)\Delta_{ij} \\
&= \sum_{ij}p^t_{ij} \frac{w_{ij}-1}{s} = \frac{\ol w^t-1}{s}.
\end{align*}
Thus 
\begin{align*}
AF_{i}^T &= \frac{1}{T}\sum_{t=1}^T \ul g^t_i = \frac{1}{T}\sum_{t=1}^T  s m^{(t)}_i + 1 = 1+ \eps  \frac{1}{T}\sum_{t=1}^T  m^{(t)}_i\\
AF_{SR}^T &= \frac{1}{T}\sum_{t=1}^T \ol w^t = \frac{1}{T}\sum_{t=1}^T (1+ s\vec m^{(t)} \cdot \vec x^{(t)})\\
& = 1+\eps \frac{1}{T}\sum_{t=1}^T \vec m^{(t)} \cdot \vec x^{(t)}. \tag{replacing $\eps=s$}
\end{align*}
Plugging in Eq.~\eqref{eq:from_kale},
\begin{align*}
AF_{SR}^T &\geq 1+\eps\frac{1}{T}\left(\sum_{t=1}^T m_i^{(t)} - \eps\sum_{t=1}^T  (m_i^{(t)})^2 - \frac{\ln(n)}{\eps}\right) \\
&\geq 1+\eps\frac{1}{T}\left(\sum_{t=1}^T m_i^{(t)} - \eps\sum_{t=1}^T  1 - \frac{\ln(n)}{\eps}\right) \\
&= \left(1 + \eps \frac{1}{T}\sum_{t=1}^T m_i^{(t)}\right) - \eps^2 - \frac{1}{T}\ln(n) \\
&= AF_{i}^T - s^2 - \ln(n)/T,
\end{align*}
as required.
\end{proof}
%\ccw{Notation (\#4): in the display above, $\vec p^{(t)}$ should be $\vec x^{(t)}$.}

We highlight that the bound on the regret of each agent (or gene)
stated in this result depends only on the algorithm used by the agent,
and not on the strategies of other agents. These may be independent,
correlated, or even chosen in an adversarial manner. For simplicity we
present the proof for two players/genes. The extension to any number
of players is immediate because the theorem bounds the regret of each
agent separately.

By taking $s$ to zero and $T$ to infinity, we get that the average cumulative regret $AF_{SR}^T - AF^T_i$ tends to zero, as stated in Chastain et al.~\shortcite{PNAS_new} (they use a more refined form of the inequality that contains the entropy of $P^t$ rather than $\ln n$). %However this does not mean a-priori that the population dynamics converges, or that that it reaches a population distribution with good fitness. 

For the RS dynamics we get something similar but not quite the same. Since $g^{t,(r)}_i$ is not exactly the expected fitness at time $t$, we get that cumulative regret is diminishing but not w.r.t. the actual average fitness $\ol w^t$. That is, the regret is determined as if the actual expected fitness of action $a_i$ is $g^{t,(r)}_i$. More formally, we get a variation of Corollary~\ref{th:regret}, where $AF^T_i = \frac{1}{T}\sum_{t=1}^T g^{t,(r)}_i$ and $AF_{RS}^T = \frac{1}{T}\sum_{t=1}^T \sum_i x^t_i g^{t,(r)}_i$.

%The result clearly holds for every player/gene.
%We provide the  proof in Appendix~\ref{apx:regret} for completeness, but note that we only need to translate the terms we use to those that appear in Kale~\shortcite{kale2007efficient}.\dcp{I assume appendix will clarify the nuance we've talked through about how talking through marginals is enough for regret bound? I'd even advocate adding a sentence to this effect here.} The selection strength comes into play when we switch from the parameter-free PW algorithm on $W$, to an $\eps$-PW algorithm on $(W-1)/s$, for $\eps=s$ (as in Chastain et al.~\shortcite{PNAS_new}).

\subsec{Convergence under weak selection}
\label{sec:converge_weak}
%We first note the following. 
%\begin{observation}
%\label{ob:Nash}
%Let $P^t= \vec x^t \times \vec y^t$ be a product distribution. Then $P^t$ is a stable state of the SR population dynamics on $W$ if and only if $(\vec x^t,\vec y^t)$ is a Nash equilibrium in $G|_P$, where $G=W$. 
%\end{observation}
%
%This correspondence is well known; see, e.g., Kleinberg et
%al.~\shortcite{kleinberg2009multiplicative}.  To see why this is true,
%note that $(\vec x,\vec y)$ is a Nash equilibrium
%of the matrix game iff player~1 is
%indifferent between all actions in the support of $\vec x$, and
%likewise for player~2 w.r.t. $\vec y$. 

In normal-form games, following strategies with bounded or diminishing regret does not, in general, guarantee convergence to a fixed point in the strategy space.\footnote{It is known that the average joint distribution over all iterations converges to the set of correlated equilibria~\cite{blum2007learning}. This is less relevant to us because we are interested in the limit of $P^t$.} For some classes of games though, much more is known. For example, if all players in a {\em potential game} apply the $\eps$-Hedge algorithm, with a sufficiently small $\eps$, then $P^t$ converges to a Nash equilibrium~\cite{kleinberg2009multiplicative}, and
almost always to a pure Nash equilibrium. Similar results have been shown for \emph{concave games}~\cite{even2009convergence}. Since identical-interest games are both potential games and concave games, and since for small $\eps$ we have that Hedge and PW are essentially the same, these results apply to our setting. This means that
{\em under each of RS and SR dynamics, the population converges to a stable state, for a sufficiently low selection strength $s$. }

This implication is not new, and has been shown independently in the evolutionary biology literature. Indeed, Nagylaki et al.~\shortcite{nagylaki1999convergence} prove that under weak selection, the population dynamics converges to a point distribution from any initial state (that is, to a Nash equilibrium of the subgame induced by the support of the initial distribution).  
Note that under weak selection, Corollary~\ref{th:regret} becomes trivial: once in a pure Nash equilibrium $(a_{i^*},b_{j^*})$ (say at time $t^*$), the optimal action of agent~1 is to keep playing $a^*_i$. Thus for any $t>t^*$, $\ol w^t = \ul g^t_{i^*}$, and the cumulative regret does not increase further.

\ignore{
\newpar{Does sex lead to increased diversity?}

\rmr{since the result on convergence to a point is already known, I think maybe removing this paragraph completely} \dcp{agree, seems to me that you  could remove this
para and the next one}
In the last part of their paper, Chastain et al.~\shortcite{ppad} show
that uniformly sampled fitness matrices  have many mixed Nash
equilibria (in fact a number exponential in the size of their
support).   By Proposition~\ref{ob:Nash}, every such mixed equilibrium is also a mixed stable
state of the population dynamics.

Chastain et al.~\shortcite{ppad} argue that the existence of many
mixed equilibria indicates that sex promotes \emph{diversity}. While
the connection between sexual reproduction and diversity is often
mentioned, the fact that convergence is almost always to a point
distribution \dcp{point to an earlier result} seems to undermine this
diversity argument.  Although many non-pure equilibria exist, the chances
of actually reaching one of them are negligible in
the current model.
}

\section{Discussion}
%We make a stronger conjecture, albeit in the restricted domain of identical interest games:  that the parameter-free PW algorithm converges to a \emph{Nash equilibrium}.
%Equivalently, we conjecture that 
%with $r=1$ the fundamental theorem does hold regardless of selection strength, at least for a two-loci
%haploid. 
%\subsec{Comparison with Chastain et al.~\shortcite{ppad}}
Chastain et al.~\shortcite{PNAS_new} extend
an interesting connection
between evolution, learning, and games 
from asexual reproduction (i.e.,
{\em replicator dynamics}) to sexual reproduction.
% In particular,
%they establish a kind of correspondence
%for  haploids with multiple loci under weak selection. 
%%
%\rmr{move this to the discussion?}
%\newpar{Revisiting the result of Chastain et al.~\shortcite{ppad}}
 The proof of Theorem~4 in Chastain et al.~\shortcite{PNAS_new} gives a formal meaning to 
this connection. Namely, that the strategy update of each player who is using \PW{1} in the  fitness game, coincides with the change in allele frequencies of the corresponding gene (under weak selection and product distributions). This relation is generalized in our Propositions~\ref{th:SR_PW} and \ref{th:RS_PW}, since for product distributions \PW{\alpha} is the same for all $\alpha$.\dcp{this sounds like the kind of comment I'd put in a footnote at the start or end of section 3.}

Chastain et al.~\shortcite{PNAS_new} also claim something stronger: that the population dynamics is \emph{precisely} the PW dynamics. The natural formal interpretation of this conclusion would be in the spirit of our Corollary~\ref{th:dynamics_cor_SR}, i.e., that allele distributions and players' strategies would coincide after any number of steps. In our case we prove this for
the marginal probabilities.
\newpar{So does Evolution coincide with MWUA?}
As we explained, the stronger conclusion of Chastain et al. only follows from their Theorem~4 
under the assumption that $P^t$ remains a product distribution. 
This is counterfactual, in that $P^{t+1}$ is in general not a product distribution (under their assumptions on the dynamic process), and thus the next step of the \PW{r} algorithm and the population dynamics % (either RS or SR) 
would not be precisely equivalent but only \emph{approximately} equivalent.\dcp{why say SR and RS here. I find confusing} The approximation becomes less accurate in each step, and in fact even under weak selection the population dynamics may diverge from the Wright manifold, or  converge  to a different outcome than the \PW{r} algorithm, as we show in Appendix~\ref{apx:differ}. Thus while the intuition of Chastain et al.~\shortcite{PNAS_new} was correct, the only way to rectify their analysis is via the more general proof without assumptions on the selection strength (even if we accept weak selection as biologically plausible). 

In particular, one cannot conclude neither from \cite{PNAS_new} nor from our results that the population after several steps is \emph{even approximately} similar to what we would get by applying MWUA on each gene. Is seems plausible that such result could be obtained with approximation bound that depends on $s/r$ (not just $s$!), but this would require a separate and careful proof.

\newpar{What does evolution maximize?} In Chastain et al.~\shortcite{PNAS_new} (Corollary~5 in the SI text), it is 
also shown that under weak
selection, ``population genetics is tantamount to each gene optimizing
at generation $t$ a quantity equal to the cumulative expected fitness
over all generations up to $t$,'' (plus the entropy). 
While this is technically correct (our Cor.~\ref{th:regret} is a restatement of this result), we feel that an unwary reader might reach the wrong impression, that this is a mathematical explanation of some guarantee on the average fitness of the population. We thus emphasize that  both  \shortcite{PNAS_new} and our paper establish only the property of \emph{diminishing regret}, which is 
already implied when $P^t$ converges to a Nash equilibrium. Players
never have regret in a Nash equilibrium, and thus the cumulative regret tends to zero after the equilibrium is played sufficiently many times. 

Thus the population dynamics cannot provide any guarantees on fitness (or on any other property) that are not already implied by an arbitrary Nash equilibrium. In the evolutionary context this means that the outcome can be as bad as the worst local maximum of the fitness matrix.
%
%Not only does this not provide convergence to a particular genotype, but it does not
%have any implications on equilibrium quality in terms of fitness or anything else. 
%It is in fact easy to construct examples where convergence is to the worst pure Nash equilibrium in the game (see Appendix~\ref{apx:regret}).
Also note that convergence is to a \emph{point distribution} (a pure Nash equilibrium, see Sec.~\ref{sec:converge_weak}), and thus its entropy is $0$ and irrelevant for the maximization claim.

\newpar{Convergence without weak selection} 
It is an open question as to what other natural conditions are sufficient to
guarantee convergence of sexual population dynamics. 
We have conducted 
simulations that 
show that convergence to a pure equilibrium
 occurs w.h.p. even without weak selection, and in fact the
convergence speed increases as selection strength $s$ (or the learning rate $\eps$) grows.
At the same time, the quality of the solution/population reached seems to be the same regardless of the selection strength/learning rate (we measured quality as the fitness of the local maximum the dynamics converged to, normalized w.r.t. the global maximum). 
Both trends are visible in Figure~\ref{fig:stats} for $8\times 5$ matrices, based on 1000 instances for each plot. Similar results are obtained with other sizes of matrices.

\begin{figure}
\begin{center}
\includegraphics[height=54mm,width=61mm]{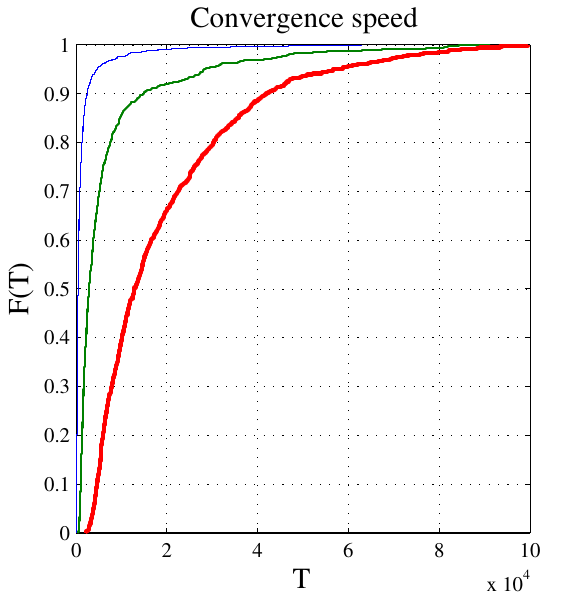}~
\includegraphics[height=54mm,width=61mm]{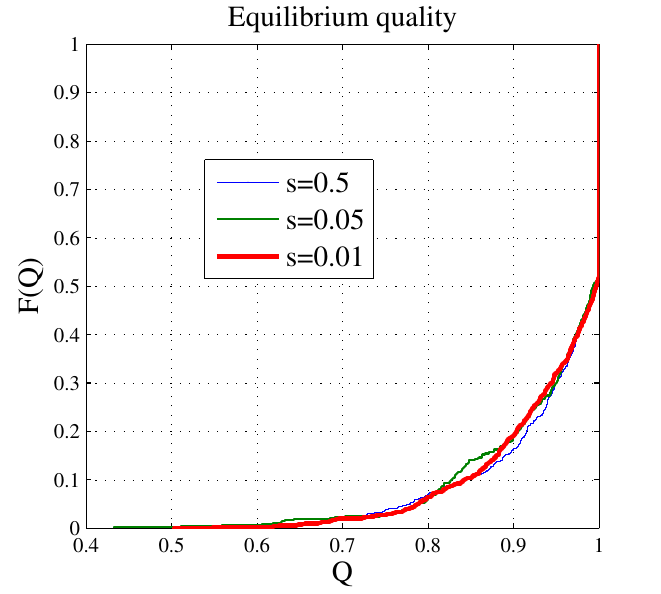}
\caption{\label{fig:stats} { On the left, $F(T)$ is the fraction of instances that converge (reach $\max_{ij} p^t_{ij}>1-10^{-5}$) within at most $T$ generations. On the right, $F(Q)$ is the fraction of instances that converge to an equilibrium with quality $q$ at most $Q$, where $q=\frac{\ol w^t-1}{\max_{ij}w_{ij}-1}$ (the ratio between average fitness in the equilibrium that was reached, and the optimal fitness).
Results are for random fitness matrices of size $8\times 5$, where $w_{ij}$ is 
sampled uniformly at random from $[1-s,1+s]$.} Note that most instances do not converge to the optimal outcome.}
\end{center}
\end{figure}

However, it is known that the sexual population dynamics on general fitness matrices (even on $4 \times 4$ matrices) does not always converge, and explicit examples have been constructed~\cite{hastings1981stable,akin1983hopf,hofbauer1984hopf}.
By Corollaries~\ref{th:dynamics_cor_SR} and \ref{th:dynamics_cor_RS}, convergence of the PW algorithm to a pure Nash equilibrium, and convergence of the population dynamics to a point distribution is the same thing. 
Thus characterizing the conditions under which these dynamics converge will answer two questions at once.

\newpar{Conclusions}

We formally describe a precise connection between population dynamics
and the multiplicative weights update algorithm. For this connection, we adopt
a version of MWUA that takes the
correlation of player strategies into account, while still supporting
no regret claims.  More specifically, two different variations of the
Polynomial Weights subclass of MWUA each coincide 
with the marginal allele distribution under the two common sexual
population dynamics (SR and RS). It is important to note that the
correspondence that we establish is between the marginal
frequencies/probabilities, rather than the full joint distribution.
%However, this is sufficient to bound the regret 
%
 %Despite a flaw in their technical argument, their main insight was correct. Namely,  the sexual
%population dynamics coincide with a variation of the MU algorithm
%that is used by players in an identical interest game. 
%
 %

Notably, weak selection is not required to make these
connections. %\footnote{The definition of mixability in \cite{PNAS} and \cite{ppad} are inconsistent, and neither of them is required for our results.}
 Yet, it is known that weak selection provides an additional
guarantee, which is  that the dynamics converge to a particular
population distribution~\cite{nagylaki1993evolution}. It remains an open question to understand what other conditions are sufficient for convergence of the PW algorithm in identical interest games. Solving this question will also
uncover more cases where the fundamental theorem of natural selection
applies.

%, but is needed if one wants to apply known results from the game theory literature to show that the population dynamics converges.

%
%What role do ``multiplicative updates'' play? Recall that in asex the frequency of a genotype $a_i b_j$ ``updates'' proportionally to its fitness, i.e. $p^{t+1}_{ij}\cong p^t_{ij} w_{ij}$. Sex (on the Wright manifold) operates in a very similar way, but on the level of the \emph{allele} rather than genotype: $x^{t+1}_i \cong x^t_i w^t_i$, where $w^t_i = \sum_j y^t_j w_{ij}$ is the \emph{average fitness} of allele $a_i$.   The recombination rate balances the two processes. A high recombination rate keeps the frequencies closer to a product distribution, and promotes good alleles, whereas low $r$ promotes good genotypes. 

%Finally, we show that $P^{t+1}$ is always a product distribution, generalizing Lemma~\ref{lemma:product}.
%\begin{proposition}
%\label{lemma:product_general}
%Let $W$ be any fitness landscape, $P^t$ be any distribution, then under sex $P^{t+1}$ is a product distribution. 
%\end{proposition}
%\begin{proof}   
%By applying Lemma~\ref{th:sex_PW_l} once on each $j\in K$, 
%\begin{align*}
%\prod_{j\in K} x^{t+1}_{i_j} & = \prod_{j\in K}\left(\frac{1}{\ol w^t} \sum_{\vec i_{-j} \in \I(-j)} p^t_{i_j,\vec i_{-j}} w_{i_j,\vec i_{-j}}
%\end{align*}
%\end{proof}

\section*{Acknowledgments}
We thank Avrim Blum, Yishay Mansour and James Zou for helpful discussions. We also acknowledge a useful correspondence with the authors of Chastain et al.~\shortcite{PNAS_new}, who clarified many points about their paper despite disagreeing with our conclusions. Any mistakes and misunderstandings remain our own.

\bibliographystyle{named}
\newpage
%\onecolumn
\appendix

\section{Extension to Multiple Genes}
\label{apx:loci}
A $k$-locus haploid has $k$ genes, each of which is inherited from one of its two parents. 
In this appendix we show how to extend our main results to a haploid with $k > 2$ loci.

\subsection{Notation}
We consider a haploid with $k$ loci, each with $n_j$ alleles, $j\leq k$.
We denote $K=[k] = \{1,2,\ldots,k\}$, and use $J$ to denote subsets of $K$. 

A genotype is defined by a vector of indices $\vec i_K=\tup{i_1,\ldots,i_k}$ for $i_j\in [n_j]$.  We denote by $\I(J)$ the set of all $\prod_{j\in J}n_j$ partial index vectors of the form $\tup{i_j}_{j\in J}$.
 We sometimes concatenate two or more partial genotypes:  $\vec i''_{J''} = \tup{\vec i_{J},\vec i'_{J'}}$ for some $\vec i_J\in\I(J), \vec i'_{J'}\in \I(J')$. We use $-J$ to denote $K\setminus J$.

The fitness of a genotype $\vec i_K$ is denoted by $w_{\vec i_K}$. $W=(w_{\vec i_K})_{\vec i_K\in \I(K)}$ is called the \emph{fitness landscape} (which is a matrix for $k=2$). Similarly, the population frequency of genotype $\vec i_K$ at time $t$ is denoted by $p_{\vec i_K}^t$, and $P^t= (p^t_{\vec i_K})_{\vec i_K\in \I(K)}$. 

The average fitness at time $t$ is
\labeq{avg_fitness}
{\ol w^t = \sum_{\vec i_K\in \I(K)} p^t_{\vec i_K} w_{\vec i_K} = \sum_{i_1=1}^{n_1}\cdots\sum_{i_k=1}^{n_k} p^t_{\vec i_K} w_{\vec i_K}.}
Let $\vec x^t_j$ be the marginal distribution of locus $j\in K$ at time $t$, i.e., for all $i_j\in [n_j]$,
$$x^t_{i_j} = \sum_{\vec i_{-j}\in \I(-j)} p^t_{i_j,\vec i_{-j}}.$$
In the special case of 2 loci, $K = \{1,2\}$, and $x^t_{i_1}, x^t_{l_2}$ correspond to $x^t_i, y^t_l$ as used in the main text.
We also define the \emph{marginal fitness} of allele $i_j\in [n_j]$ at time $t$ as the average fitness of all the population with allele $i_j$. That is, 
\labeq{marg}
{w^t_{i_j} = \sum_{\vec i_{-j}\in \I(-j)} p^t(\vec i_{-j}|i_j) w_{i_j,\vec i_{-j}}.}
%Note that for the special case of $k=2$, $p_{i_1,i_2} = x_{i_1}y_{i_2}$, we have that 
%\labeq{marginal_two}
%{w^t_{i_1} = \sum_{i_2 =1}^{n_2} p^t(i_2|i_1) w_{i_1,i_2} = \sum_{i_2 =1}^{n_2} y^t_{i_2} w_{i_1,i_2}.}

%Due to space constraints, we only bring here the derivation for the RS dynamics.  The derivation for the (multi-dimensional) SR dynamics is considerably lengthier and is thus omitted.
\subsection{RS dynamics}

According to the multi-dimensional extension of $p^{t+1}_{\vec i_K}$, 
\labeq{RS_multi}
{p^{t+1}_{\vec i_K} = \frac{1}{\ol w^R} (r  w_{\vec i_K} \prod_{j\in K}x^t_{i_j} + (1-r) w_{\vec i_K} p^t_{\vec i_K}).}
\ccw{Recombination model (\#3): for $k>2$ this is a full-shuffle model (each locus of a recombinant drawn independently from the population marginals), not the two-parent crossover of Eq.~\eqref{eq:pop}. The two agree only for $k=2$. Either say so, or switch Eq.~\eqref{eq:pop} to the same model: $p^{t+1}_{\vec i_K}=r\prod_{j\in K}x^{S}_{i_j}+(1-r)\frac{w_{\vec i_K}p^t_{\vec i_K}}{\ol w^t}$ with $x^S_{i_j}=\frac{1}{\ol w^t}\sum_{\vec i_{-j}}p^t_{i_j,\vec i_{-j}}w_{i_j,\vec i_{-j}}$. For $k=2$ that is exactly Eq.~\eqref{eq:sex_SR_r}, and Lemma~\ref{lemma:SR} becomes a two-line proof, since both terms have locus-$j$ marginal $x^S_{i_j}$.}

Given a game $G$ and a joint distribution $P^t$, let $\ol g^t_{i_j}=\sum_{\vec i_{-j} \in \I(-j)} (\prod_{j'\in \I(-j)}x^t_{i_{j'}}) g_{i_j,\vec i_{-j}}$. That is, the expected utility of playing $a_{i_j}$ when every agent $j'$ independently plays $\vec x^t_{j'}$.

\begin{lemma} 
\label{lemma:RS}
 Let $W$ be any fitness matrix, and consider the game $G=W$. Then under the RS population dynamics, for any distribution $P^t$ and any $r\in[0,1]$,
$$x^{t+1}_{i_j} = \frac{1}{\ol w^R} x^t_{i_j} g^{t,(r)}_{i_j}.$$
\end{lemma}
\begin{proof}
\begin{align*}
x^{t+1}_{i_j} &= \sum_{\vec i_{-j} \in \I(-j)} p^{t+1}_{i_j,\vec i_{-j}} \tag{By definition} \\
&=\sum_{\vec i_{-j} \in \I(-j)}\frac{1}{\ol w^R} (r  w_{i_j,\vec i_{-j}} x^t_{i_j}\prod_{j'\in \cal I(-j)}x^t_{i_{j'}} + (1-r) w_{i_j,\vec i_{-j}} p^t_{i_j,\vec i_{-j}})  \tag{By Eq.~\eqref{eq:RS_multi}} \\
&=\frac{1}{\ol w^R}x^t_{i_j}\left(r\sum_{\vec i_{-j} \in \I(-j)} w_{i_j,\vec i_{-j}}\prod_{j'\in \cal I(-j)}x^t_{i_{j'}} \right. \\
~~~~&+ \left. (1-r)\sum_{\vec i_{-j} \in \I(-j)}w_{i_j,\vec i_{-j}} P^t(\vec i_{-j}|i_j)\right)\\
&=\frac{1}{\ol w^R}x^t_{i_j}\left(r \ol g^t_i + (1-r)\ul g^t_i\right) = \frac{1}{\ol w^R}x^t_{i_j} g^{t,(r)}_i. 
\end{align*}
\end{proof}

%
%
%\subsection{Equilibrium quality}
%Consider a fitness matrix $W$ where $w_{11}=1+s$, $w_{22}=1+s/D$, for some large number $D$, and the two other entries are $1$. The corresponding identical interest game has two pure Nash equilibria.
%
%If we start from an initial product distribution where $x^0_1,y^0_1$ are sufficiently small (say, of order $1/D$), then $P^t$ converges to $P^*$ where $p^*_{22}=1$, i.e. to the worse Nash equilibrium. 
%
%According to Chastain et al., the mixability of allele $a_i$ is
%$$m^t_x(i) = \sum_{j'} y^t_{j'} \Delta_{ij'} = \sum_{j'} y^t_{j'} (w_{ij'}-1)/s = \frac{1}{s}\sum_{j'} y^t_{j'} w_{ij'} - 1/s,$$
%and likewise for $b_j$. Thus after convergence to $P^*$ we have $m^*_x(2)= m^*_y(2)=\frac{1}{s}( 0  +  1+s/D) - 1/s = 1/D$. If we run $T$ steps, cumulative mixability of the pair $\tup{a_2,b_2}$ approaches $T/D$. 
%
%By changing only one of the alleles we cannot do better. However, if we compare to the cumulative mixability of the pair $\tup{a_1,a_1}$, corresponding to the distribution $P'$ where $p'_{11} = 1$, we have that $m'_x(1)= m'_y(1)=\frac{1}{s}(1+s + 0) - 1/s = 1$. After $T$ steps the cumulative mixability approaches $T$. Thus the mixability of the pair $\tup{a_1,b_1}$ becomes arbitrarily better than the mixability of the pair $\tup{a_2,b_2}$ as we increase $D$.

\subsection{SR dynamics}
The SR population dynamics under sexual reproduction is defined as:
\labeq{pop}
{p^{t+1}_{\vec i_K} = 2r\sum_{\vec i'_K\in \I(K)} \frac{1}{2^k}\sum_{J \subseteq K}\frac{p^t_{\vec i_J,\vec i'_{-J}} w_{\vec i_J,\vec i'_{-J}} p^t_{\vec i'_J,\vec i_{-J}} w_{\vec i'_J,\vec i_{-J}}}{(\ol w^t)^2}  + (1-2r)\frac{w_{\vec i_K} p_{\vec i_K}}{\ol w^t}.}
We can think of $J$ as the set of genes that are inherited  from the ``first'' parent, and $-J$ as the set of genes that are inherited form the ``second'' parent. Thus a possible genotype of the offspring of parents with genotypes $\vec i, \vec i'$ is $\tup{\vec i_J,\vec i'_{-J}}$. 

For $k=2$, the cases $J=\emptyset$ and $J=K=\{1,2\}$ yield $r \cdot p^{(S)}_{ij}$ in the left term which is added to the $(1-2r)p^{(S)}_{ij}$ in the right term, so Eq.~\eqref{eq:pop} coincides with Eq.~\eqref{eq:sex_SR_r} from the main text. 

%\ccw{Inconsistent with the main text (\#3): the sum runs over all $J\subseteq K$, including $J=\emptyset$ and $J=K$, which return a copy of a parent. So for $k=2$ this equation with rate $r$ equals Eq.~\eqref{eq:sex_SR_r} with rate $r/2$. Minimal fix: replace $\frac{1}{2^k}\sum_{J\subseteq K}$ by $\frac{1}{2^k-2}\sum_{\emptyset\neq J\subsetneq K}$, here and in the first two lines of the proof of Lemma~\ref{lemma:SR}. Then $r$ is again the probability of a recombinant offspring, and for $k=2$ this is exactly Eq.~\eqref{eq:sex_SR_r}. Lemma~\ref{lemma:SR} holds for any distribution over $J$, because each $J$ term separately equals $\frac{1}{\ol w^t}\sum_{\vec i_{-j}}p^t_{i_j,\vec i_{-j}}w_{i_j,\vec i_{-j}}$; only the final counting step of the proof changes ($2^k-2$ terms of weight $\frac{1}{2^k-2}$). Checked numerically for $k=3$.}

%We start by providing a generalization of Lemma~1.

%\medskip
%\noindent\textbf{Lemma~1$^*$.}\begin{itshape}
\begin{lemma}
\label{lemma:SR}
Let $W$ be any fitness landscape, then under the SR dynamics, $$x^{t+1}_{i_j} = \frac{1}{\ol w^t} \sum_{\vec i_{-j} \in \I(-j)} p^t_{i_j,\vec i_{-j}} w_{i_j,\vec i_{-j}}.$$ 
\end{lemma}
%\end{itshape}

\begin{proof}
Let $J^* = J\cup \{j\}$, $-J^*=K\setminus (J \cup \{j\})$. 
%\begin{small}

\begin{align*}
x^{t+1}_{i_j} &= \sum_{\vec i_{-j} \in \I(-j)} p^{t+1}_{i_j,\vec i_{-j}} \tag{By definition} \\
=& \sum_{\vec i_{-j} \in \I(-j)} (2r\sum_{\vec i'_K\in \I(K)} \frac{1}{2^k}\sum_{J \subseteq K}\frac{1}{(\ol w^t)^2} p^t_{\vec i_J,\vec i'_{-J}} w_{\vec i_J,\vec i'_{-J}} p^t_{\vec i'_J,\vec i_{-J}} w_{\vec i'_J,\vec i_{-J}}\\
&~~~+ (1-2r)\frac{p^t_{i_j,\vec i_{-j}} w_{i_j,\vec i_{-j}}}{\ol w^t})
\tag{By Eq.~\eqref{eq:pop}}\\
=&2r\frac{1}{2^k}\sum_{J \subseteq K}\sum_{\vec i_{-j} \in \I(-j)} \sum_{\substack { \vec i'_J\in \I(J) \\ \vec i'_{-J}\in \I(-J)}} \frac{1}{(\ol w^t)^2} p^t_{\vec i_J,\vec i'_{-J}} w_{\vec i_J,\vec i'_{-J}} p^t_{\vec i'_J,\vec i_{-J}} w_{\vec i'_J,\vec i_{-J}}\\
&+(1-2r)\sum_{\vec i_{-j} \in \I(-j)}\frac{p^t_{i_j,\vec i_{-j}} w_{i_j,\vec i_{-j}}}{\ol w^t}\\
=&2r C + (1-2r) D
\end{align*}

We first analyze part $C$:
\begin{align*}
C=&  \frac{1}{2^k(\ol w^t)^2} \sum_{J \subseteq -\{j\}}\sum_{\substack{ \vec i_{J} \in \I(J) \\ \vec i_{-J^*} \in \I(-J^*)}}
\sum_{\substack{\vec i'_J\in \I(J) \\ \vec i'_{-J}\in \I(-J)}} p^t_{\vec i_J,\vec i'_{-J}} w_{\vec i_J,\vec i'_{-J}} p^t_{\vec i'_J,\vec i_{-J}} w_{\vec i'_J,\vec i_{-J}}.\\
 &+ \sum_{J \subseteq -\{j\}}\sum_{\substack{\vec i_{J} \in \I(J)\\ \vec i_{-J^*} \in \I(-J^*) }}
\sum_{\substack{ \vec i'_{J^*}\in \I(J^*) \\ \vec i'_{-J^*}\in \I(-J^*)}}
 p^t_{\vec i_{J^*},\vec i'_{-J^*}} w_{\vec i_{J^*},\vec i'_{-J^*}} p^t_{\vec i'_{J^*},\vec i_{-J^*}} w_{\vec i'_{J^*},\vec i_{-J^*}} \\
=&  \frac{1}{2^k(\ol w^t)^2} \sum_{J \subseteq -\{j\}}\left( \sum_{\vec i_{-J^*} \in \I(-J^*)}\sum_{\vec i'_J\in \I(J)} p^t_{\vec i'_J,\vec i_{-J}} w_{\vec i'_J,\vec i_{-J}} \sum_{\vec i_{J} \in \I(J)}\sum_{\vec i'_{-J}\in \I(-J)} p^t_{\vec i_J,\vec i'_{-J}} w_{\vec i_J,\vec i'_{-J}} \right.\\
 &\left. + \sum_{\vec i_{J} \in \I(J)} \sum_{\vec i'_{-J^*}\in \I(-J^*)} p^t_{\vec i_{J^*},\vec i'_{-J^*}} w_{\vec i_{J^*},\vec i'_{-J^*}} \sum_{\vec i'_{J^*}\in \I(J^*)}\sum_{\vec i_{-J^*} \in \I(-J^*)} p^t_{\vec i'_{J^*},\vec i_{-J^*}} w_{\vec i'_{J^*},\vec i_{-J^*}}\right) \\
=&  \frac{1}{2^k(\ol w^t)^2} \sum_{J \subseteq -\{j\}}\left( \sum_{\vec i_{-J^*} \in \I(-J^*)}\sum_{\vec i'_J\in \I(J)} p^t_{\vec i'_J,\vec i_{-J}} w_{\vec i'_J,\vec i_{-J}} \sum_{\vec i''_{K} \in \I(K)} p^t_{\vec i''_K} w_{\vec i''_K} \right.\\
 &\left. + \sum_{\vec i_{J} \in \I(J)} \sum_{\vec i'_{-J^*}\in \I(-J^*)} p^t_{\vec i_{J^*},\vec i'_{-J^*}} w_{\vec i_{J^*},\vec i'_{-J^*}} \sum_{\vec i''_{K}\in \I(K)} p^t_{\vec i''_{K}} w_{\vec i''_{K}}\right) \\
=& \frac{1}{2^k(\ol w^t)^2} \sum_{J \subseteq -\{j\}}\left( \sum_{\vec i_{-J^*} \in \I(-J^*)}\sum_{\vec i'_J\in \I(J)} p^t_{\vec i'_J,\vec i_{-J}} w_{\vec i'_J,\vec i_{-J}} \ol w^t \right.\\
 &\left. + \sum_{\vec i_{J} \in \I(J)} \sum_{\vec i'_{-J^*}\in \I(-J^*)} p^t_{\vec i_{J^*},\vec i'_{-J^*}} w_{\vec i_{J^*},\vec i'_{-J^*}} \ol w^t\right) \tag{By Eq.~\eqref{eq:avg_fitness}}\\
=& \frac{1}{\ol w^t} \frac{1}{2^k}\sum_{J \subseteq -\{j\}}\left( \sum_{\vec i_{-j} \in \I(-j)} p^t_{i_j,\vec i_{-j}} w_{i_j,\vec i_{-j}} +\sum_{\vec i_{-j} \in \I(-j)} p^t_{i_j,\vec i_{-j}} w_{i_j,\vec i_{-j}} \right) \\
=& \frac{1}{\ol w^t}\sum_{\vec i_{-j} \in \I(-j)} p^t_{i_j,\vec i_{-j}} w_{i_j,\vec i_{-j}} \frac{1}{2^{k-1}}\sum_{J \subseteq -\{j\}}1 \\
=&  \frac{1}{\ol w^t} \sum_{\vec i_{-j} \in \I(-j)} p^t_{i_j,\vec i_{-j}} w_{i_j,\vec i_{-j}}.
\end{align*}
%\end{small}
\end{proof}

Finally, 
\begin{align*}
x^{t+1}_{i_j} &= 2rC+(1-2r)D = 2r \frac{1}{\ol w^t} \sum_{\vec i_{-j} \in \I(-j)} p^t_{i_j,\vec i_{-j}} w_{i_j,\vec i_{-j}} + (1-2r)\sum_{\vec i_{-j} \in \I(-j)}\frac{p^t_{i_j,\vec i_{-j}} w_{i_j,\vec i_{-j}}}{\ol w^t} \\
&= \frac{1}{\ol w^t} \sum_{\vec i_{-j} \in \I(-j)} p^t_{i_j,\vec i_{-j}} w_{i_j,\vec i_{-j}}.
%\qedhere 
\end{align*}

As in the case of two loci, we use the lemma to show that under the SR dynamics, the marginal distribution of gene $j\in K$ develops as if gene $j$ is applying the PW algorithm. 

Given a game $G$ and a joint distribution $P^t$, let $\ul g^t_{i_j}=\sum_{\vec i_{-j} \in \I(-j)} P^t(\vec i_{-j}|i_j) g_{i_j,\vec i_{-j}}$. That is, the expected utility to $j$ of using the pure action $a_{i_j}$ at time $t$.

\begin{proposition}
\label{th:SR_PW2_gen}
Let $W$ be any fitness matrix, and consider the game $G=W$. Then under the SR population dynamics, for any distribution $P^t$ and any $r\in[0,1]$,
we have
\labeq{sex_PW_gen}
{x^{t+1}_{i_j} = \frac{1}{\ol w^t} x^t_{i_j} \ul g^t_{i_j}.} % (\sum_j y^t_j w_{ij}).}
\end{proposition}

\begin{proof} Applying Lemma~\ref{lemma:SR},
%\ccw{Proof slip (\#4): in the first line the conditional belongs on $x^t_{i_j}$, not on $p^t$: $\sum_{\vec i_{-j}}p^t_{i_j,\vec i_{-j}}w_{i_j,\vec i_{-j}}=\sum_{\vec i_{-j}}x^t_{i_j}P^t(\vec i_{-j}|i_j)w_{i_j,\vec i_{-j}}$. The second expression then drops the conditional altogether. The last line and the result are right. Also, in the statement $\ul g^t_i$ should be $\ul g^t_{i_j}$.}
%\begin{small}
 \begin{align*}
x^{t+1}_{i_j}=&  \frac{1}{\ol w^t} \sum_{\vec i_{-j} \in \I(-j)} p^t_{i_j,\vec i_{-j}}w_{i_j,\vec i_{-j}} 
= \frac{1}{\ol w^t} \sum_{\vec i_{-j} \in \I(-j)}  x^t_{i_j}  P^t(\vec i_{-j}|i_j)   w_{i_j,\vec i_{-j}}\\
=& \frac{1}{\ol w^t} x^{t}_{i_j} \sum_{\vec i_{-j} \in \I(-j)} P^t(\vec i_{-j}|i_j)w_{i_j,\vec i_{-j}}
= \frac{1}{\ol w^t} x^{t}_{i_j} \ul g^t_{i_j} \tag{By Eq.~\eqref{eq:marg}}.
\end{align*}
%\end{small}
\end{proof}
 The multi-dimensional extension of Corollary~\ref{th:dynamics_cor_SR} follows in the same way from Proposition~\ref{th:SR_PW2_gen}.

\section{PW and product distributions}
\label{apx:differ}
Consider the ``uncorrelated'' version of the PW algorithm, which is the one used in \cite{PNAS_new}:
\labeq{indi}
{x^{t+1}_i = x^t_i \sum_j y^t_j w_{ij} = x^t_i \ol g^t_i.}

In \cite{PNAS_new} there is no distinction between RS and SR. The formal definition that they use  coincides with RS (p.~1 of the SI text), whereas in an earlier draft they used SR (p.5 in \cite{ppad}). 
%for $r=1$, 
In a private communication the authors clarified that they use SR and RS interchangeably, since under weak selection they are very close. %Also, that their analysis applies for any $r\in[0,0.5]$.

%\paragraph{Difference in the second step}
%Consider a $2\times 2$ fitness matrix where $w_{11}= 1+s$ for some $s>0$ and $w_{ij}=1$ otherwise.
%Let $P^0$ be the uniform distribution, and suppose that $r=0$ (although we would get a similar example for any $r<1$).  Then under Eq.~\eqref{eq:indi}, we 
%have $\vec x^1_{PW} = \vec y^1_{PW} = (1/2+s/4,1/2-s/4)$, and thus 
%$$P_{PW}^{1} = \vec x^1_{PW} \times \vec y^1_{PW} = \left(\begin{array}{cc}4+4s +s^2& 4 -s^2  \\ 4 - s^2 & 4-4s+s^2\end{array}\right).$$
 %On the other hand, under each of the SR and RS dynamics we get that $P^1 \cong W \neq P_{PW}^1$. Further, while the marginals of $P^1$ are still equal to those of $P_{PW}^1$, this is no longer true for $P^2$ and $P^2_{PW}$:
%$$\vec x^2_{PW} = (1/2+s/4 + s^2/8 + s^3/16, 1/2 - s/4 - s^2/8 - s^3/16),$$
%but under population dynamics
%$$P^2 \cong \left(\begin{array}{cc}1+2s +s^2& 1  \\ 1 & 1 \end{array}\right),\ \vec x^2 =  (2+2s+s^2,2)/(
%4+2s+s^2) \neq \vec x^2_{PW}.$$

\paragraph{Divergence from the Wright manifold}
Chastain et al.~\shortcite{PNAS_new} justify the assumption that $P^t$ is a product distribution by quoting the result of Nagylaki~\shortcite{nagylaki1993evolution}, which states for any process $(P^t)_{t}$ there is a ``corresponding process'' on the Wright manifold, which converges to the same point. However the authors do not explain why this corresponding process is the one they assume in their paper. To further stress this point, we will show that the population dynamics and the PW algorithm used in \cite{PNAS_new} can significantly differ (we saw empirically that the marginals also differ significantly). 

Consider the $2\times 2$ fitness matrix where $w_{11} = 1+s$, and $w_{ij}=1$ otherwise. For simplicity assume first that $r=0$ (thus SR and RS are the same). Suppose that $P^0$ is the uniform distribution (that is on the Wright manifold). While the population dynamics will eventually converge to $p_{11}=1$, there is some $t$ s.t. $P^t$ is approximately $\left(\begin{array}{cc}5/8 & 1/8  \\ 1/8 & 1/8 \end{array}\right)$. 
Thus $\left\|P^t - \vec x^t \times \vec y^t\right\|_d \geq \frac18 = \Omega(1)$ for any $\ell_d$ norm and regardless of the selection strength $s$.

The gap is still large for other small values of  $r$ (see Fig.~\ref{fig:wright}).   Thus the population dynamics can get very far from the Wright manifold. \rmr{check this}

\begin{figure}
    \centering
    % TikZ figure: Wright-manifold distance in App. B example 1 (W=[[1+s,1],[1,1]], s=0.01, uniform P^0, SR)
% For: replacing/supporting the "Divergence from the Wright manifold" paragraph
% Target: evolution.appendix.arxiv.tex, App. B (apx:differ), paragraph "Divergence from the Wright manifold"
% Created: 2026-10-21   Status: draft
% Not referenced by the build until merged.
%
% Appendix B, example 1: W = [[1+s,1],[1,1]], s = 0.01, uniform P^0, SR dynamics
% (main-text convention: with prob. r the offspring is fully recombinant).
% D^t = p_11 p_22 - p_12 p_21 = ||P^t - P_m^t||_infty (every entry of P^t - P_m^t is +-D^t).
% Left: max_t D^t vs r in [0,0.5].  Right: D^t vs t for r = 0.001 (run until p_11 > 1-1e-7).
% Requires \usepackage{pgfplots} (\pgfplotsset{compat=1.17} or later).
\begin{tikzpicture}
\begin{axis}[name=left, width=6.5cm, height=4.6cm,
  xlabel={$r$}, ylabel={$\max_t \|P^t-\bm x^t \bm y^t\|_\infty$},
  xmin=0, xmax=0.5, ymode=log, ymin=4e-4, ymax=0.1,
  grid=major, grid style={gray!20}, axis lines=left, tick align=outside,
  every axis plot/.append style={thick}]
\addplot[blue!70!black, no marks] coordinates {(0,0.0625) (0.001,0.057914) (0.002,0.053887) (0.003,0.05033) (0.004,0.047169) (0.005,0.044346) (0.006,0.041811) (0.007,0.039525) (0.008,0.037455) (0.009,0.035573) (0.01,0.033856) (0.02,0.022521) (0.03,0.016633) (0.04,0.013073) (0.05,0.010705) (0.06,0.0090213) (0.07,0.0077661) (0.08,0.0067957) (0.09,0.0060237) (0.1,0.0053955) (0.11,0.0048747) (0.12,0.0044359) (0.13,0.0040614) (0.14,0.0037381) (0.15,0.0034562) (0.16,0.0032083) (0.17,0.0029886) (0.18,0.0027926) (0.19,0.0026166) (0.2,0.0024578) (0.21,0.0023137) (0.22,0.0021825) (0.23,0.0020624) (0.24,0.0019521) (0.25,0.0018505) (0.26,0.0017566) (0.27,0.0016695) (0.28,0.0015885) (0.29,0.0015131) (0.3,0.0014426) (0.31,0.0013766) (0.32,0.0013147) (0.33,0.0012565) (0.34,0.0012016) (0.35,0.0011499) (0.36,0.001101) (0.37,0.0010548) (0.38,0.0010109) (0.39,0.00096929) (0.4,0.00092975) (0.41,0.00089211) (0.42,0.00085622) (0.43,0.00082203) (0.44,0.00078936) (0.45,0.00075812) (0.46,0.00072825) (0.47,0.00069964) (0.48,0.00067221) (0.49,0.00064588) (0.5,0.00062063)};
\addplot[red!70!black, only marks, mark=*, mark size=1.6pt] coordinates {(0.001,0.057914)};
%\node[anchor=west, font=\footnotesize, red!70!black] at (axis cs:0.012,0.0579) {$r=0.001$};
\end{axis}~~~~~~~~
\begin{axis}[at={(left.south east)}, anchor=south west, xshift=1.6cm,
  width=6.5cm, height=4.6cm,
  xlabel={$t$}, ylabel={$\|P^t-\bm x^t\bm y^t\|_\infty$},
  xmin=0, xmax=1200, ymin=0, title={$r=0.001$},
  grid=major, grid style={gray!20}, axis lines=left, tick align=outside,
  scaled x ticks=false, scaled y ticks=false, yticklabel style={/pgf/number format/fixed, /pgf/number format/precision=2}, xticklabel style={/pgf/number format/1000 sep={}}]
\addplot[red!70!black, thick, no marks] coordinates {(0,0) (5,0.003099) (10,0.0061757) (15,0.0092232) (20,0.012234) (25,0.015202) (30,0.018119) (35,0.020978) (40,0.023772) (45,0.026494) (50,0.029136) (55,0.031692) (60,0.034156) (65,0.036521) (70,0.038782) (75,0.040932) (80,0.042966) (85,0.044881) (90,0.046672) (95,0.048336) (100,0.049869) (105,0.051269) (110,0.052535) (115,0.053666) (120,0.054661) (125,0.05552) (130,0.056244) (135,0.056834) (140,0.057293) (145,0.057624) (150,0.057829) (155,0.057911) (156,0.057914) (160,0.057876) (165,0.057728) (170,0.057471) (175,0.05711) (180,0.056652) (185,0.056102) (190,0.055466) (195,0.05475) (200,0.053959) (205,0.053101) (210,0.052181) (215,0.051205) (220,0.05018) (225,0.049111) (230,0.048004) (235,0.046865) (240,0.045699) (245,0.044511) (250,0.043305) (255,0.042088) (260,0.040862) (265,0.039633) (270,0.038404) (275,0.037178) (280,0.03596) (285,0.034751) (290,0.033555) (295,0.032374) (300,0.031211) (305,0.030067) (310,0.028944) (315,0.027843) (320,0.026767) (325,0.025716) (330,0.024692) (335,0.023693) (340,0.022723) (345,0.02178) (350,0.020866) (355,0.019979) (360,0.019121) (365,0.018292) (370,0.017491) (375,0.016717) (380,0.015972) (385,0.015253) (390,0.014561) (395,0.013896) (400,0.013257) (405,0.012642) (410,0.012053) (415,0.011487) (420,0.010945) (425,0.010425) (430,0.0099275) (435,0.0094512) (440,0.0089955) (445,0.0085598) (450,0.0081434) (455,0.0077456) (460,0.0073656) (465,0.007003) (470,0.0066569) (475,0.0063268) (480,0.006012) (485,0.0057119) (490,0.005426) (495,0.0051536) (500,0.0048942) (505,0.0046471) (510,0.004412) (515,0.0041882) (520,0.0039753) (525,0.0037728) (530,0.0035802) (535,0.0033971) (540,0.003223) (545,0.0030576) (550,0.0029003) (555,0.002751) (560,0.0026091) (565,0.0024743) (570,0.0023463) (575,0.0022247) (580,0.0021093) (585,0.0019998) (590,0.0018958) (595,0.0017971) (600,0.0017035) (605,0.0016146) (610,0.0015303) (615,0.0014504) (620,0.0013745) (625,0.0013025) (630,0.0012343) (635,0.0011696) (640,0.0011082) (645,0.00105) (650,0.00099487) (655,0.00094256) (660,0.00089298) (665,0.00084597) (670,0.00080142) (675,0.00075919) (680,0.00071917) (685,0.00068123) (690,0.00064529) (695,0.00061122) (700,0.00057894) (705,0.00054835) (710,0.00051937) (715,0.00049191) (720,0.00046589) (725,0.00044124) (730,0.00041788) (735,0.00039576) (740,0.0003748) (745,0.00035495) (750,0.00033614) (755,0.00031832) (760,0.00030145) (765,0.00028546) (770,0.00027032) (775,0.00025598) (780,0.0002424) (785,0.00022954) (790,0.00021735) (795,0.00020582) (800,0.00019489) (805,0.00018454) (810,0.00017474) (815,0.00016545) (820,0.00015666) (825,0.00014834) (830,0.00014046) (835,0.00013299) (840,0.00012592) (845,0.00011923) (850,0.00011289) (855,0.00010689) (860,0.0001012) (865,9.5822e-05) (870,9.0726e-05) (875,8.59e-05) (880,8.1331e-05) (885,7.7004e-05) (890,7.2908e-05) (895,6.9029e-05) (900,6.5356e-05) (905,6.1878e-05) (910,5.8585e-05) (915,5.5468e-05) (920,5.2516e-05) (925,4.9721e-05) (930,4.7074e-05) (935,4.4568e-05) (940,4.2196e-05) (945,3.995e-05) (950,3.7823e-05) (955,3.5809e-05) (960,3.3903e-05) (965,3.2098e-05) (970,3.0389e-05) (975,2.8771e-05) (980,2.7239e-05) (985,2.5788e-05) (990,2.4415e-05) (995,2.3115e-05) (1000,2.1884e-05) (1005,2.0719e-05) (1010,1.9615e-05) (1015,1.8571e-05) (1020,1.7582e-05) (1025,1.6645e-05) (1030,1.5759e-05) (1035,1.4919e-05) (1040,1.4125e-05) (1045,1.3372e-05) (1050,1.266e-05) (1055,1.1986e-05) (1060,1.1347e-05) (1065,1.0743e-05) (1070,1.0171e-05) (1075,9.629e-06) (1080,9.1161e-06) (1085,8.6305e-06) (1090,8.1708e-06) (1095,7.7355e-06) (1100,7.3234e-06) (1105,6.9333e-06) (1110,6.564e-06) (1115,6.2143e-06) (1120,5.8832e-06) (1125,5.5698e-06) (1130,5.2731e-06) (1135,4.9922e-06) (1140,4.7262e-06) (1145,4.4744e-06) (1150,4.2361e-06) (1155,4.0104e-06) (1160,3.7967e-06) (1165,3.5945e-06) (1170,3.403e-06) (1175,3.2217e-06) (1180,3.05e-06) (1185,2.8875e-06) (1190,2.7337e-06) (1195,2.588e-06) (1200,2.4502e-06) (1205,2.3196e-06) (1210,2.196e-06) (1215,2.079e-06) (1220,1.9683e-06) (1225,1.8634e-06) (1230,1.7641e-06) (1235,1.6701e-06) (1240,1.5812e-06) (1245,1.4969e-06) (1250,1.4172e-06) (1255,1.3416e-06) (1260,1.2702e-06) (1265,1.2025e-06) (1270,1.1384e-06) (1275,1.0778e-06) (1280,1.0203e-06) (1285,9.6598e-07) (1290,9.1451e-07) (1295,8.6579e-07) (1300,8.1966e-07) (1305,7.7599e-07) (1310,7.3464e-07) (1315,6.955e-07) (1320,6.5844e-07) (1325,6.2336e-07) (1330,5.9015e-07) (1335,5.5871e-07) (1340,5.2894e-07) (1345,5.0075e-07) (1350,4.7407e-07) (1355,4.4881e-07) (1360,4.249e-07) (1365,4.0226e-07) (1370,3.8083e-07) (1375,3.6054e-07) (1380,3.4133e-07) (1385,3.2314e-07) (1390,3.0592e-07) (1395,2.8962e-07) (1400,2.7419e-07) (1405,2.5958e-07) (1410,2.4575e-07) (1415,2.3266e-07) (1420,2.2026e-07) (1425,2.0853e-07) (1430,1.9742e-07) (1435,1.869e-07) (1440,1.7694e-07) (1445,1.6751e-07) (1450,1.5859e-07) (1455,1.5014e-07) (1460,1.4214e-07) (1465,1.3456e-07) (1470,1.2739e-07) (1475,1.2061e-07) (1480,1.1418e-07) (1485,1.081e-07) (1490,1.0234e-07) (1495,9.6884e-08) (1500,9.1722e-08) (1505,8.6835e-08) (1510,8.2208e-08) (1515,7.7828e-08) (1520,7.3681e-08) (1525,6.9755e-08) (1530,6.6038e-08) (1535,6.252e-08) (1540,5.9189e-08) (1545,5.6035e-08) (1550,5.3049e-08) (1555,5.0223e-08) (1560,4.7547e-08) (1565,4.5013e-08) (1570,4.2615e-08) (1575,4.0344e-08) (1580,3.8195e-08) (1585,3.616e-08) (1590,3.4233e-08) (1595,3.2409e-08) (1600,3.0682e-08) (1605,2.9047e-08) (1610,2.75e-08) (1615,2.6034e-08) (1620,2.4647e-08) (1625,2.3334e-08) (1630,2.2091e-08) (1635,2.0914e-08) (1640,1.9799e-08) (1645,1.8744e-08) (1650,1.7746e-08) (1655,1.68e-08) (1660,1.5905e-08) (1665,1.5058e-08) (1670,1.4255e-08) (1675,1.3496e-08) (1680,1.2777e-08) (1685,1.2096e-08) (1690,1.1451e-08) (1695,1.0841e-08) (1700,1.0264e-08) (1705,9.7167e-09) (1710,9.199e-09) (1715,8.7088e-09) (1720,8.2448e-09) (1725,7.8055e-09) (1730,7.3896e-09) (1735,6.9959e-09) (1740,6.6231e-09) (1745,6.2702e-09) (1750,5.9361e-09) (1755,5.6198e-09) (1757,5.4981e-09)};
\end{axis}
\end{tikzpicture} 
    \caption{Left: the maximal distance from the Wright manifold as a function of $r$. Right: The distance from the Wright manifold over time for $r=0.001$. }
    \label{fig:wright}
\end{figure}

In the example above both processes will converge to the same outcome ($p_{11}=1$), but at different rates.

\paragraph{Evolution does not coincide with MWUA}
We can think about the `PW dynamics' as repeatedly applying PW on each marginal $\vec x^0, \vec y^0$. This, in general, \emph{does not} coincide with either the SR or RS dynamics, and may get arbitrarily far even for small $s$. 

%In fact, the dynamics may even converge to different outcomes. For example, for $s=0.01$ consider $W = \left(\begin{array}{cc}1.01& 1  \\ 1 & 1.0099603 \end{array}\right)$. If the initial distribution is $\vec x^0 = \vec y^0 = (0.499,0.501)$, then the (independent) PW dynamics converges to $p_{22}=1$, whereas for $r=0.5$ the SR dynamics converges to $p_{11}=1$. Such examples can be constructed for any values of $s>0$  and $r<1$.

In fact, the dynamics may even converge to different outcomes. For example, for $s=0.01$ consider $W = \left(\begin{array}{cc}1.01& 1  \\ 1 & 1.005 \end{array}\right)$. If the initial distribution is $\vec x^0 = \vec y^0 = (0.32,0.68)$, then the (independent) PW dynamics converges to $p_{22}=1$, whereas for $r=0.05$ both the SR and RS dynamics converge to $p_{11}=1$. We have been able to craft similar examples for other values of $s>0$  and $r\in [0,0.5]$, although these are less frequent and less robust as $s/r$ gets smaller.

\bibliography{fp231-meir}
%
%\newpage

\end{document}